%% file: 2026_P14_Voting_Classifiers_under_Graph_Mixing.tex
\documentclass[letterpaper]{article}
\usepackage[preprint]{aaai2027}
\usepackage[hyphens]{url}
\usepackage{graphicx}
\usepackage{natbib}
\usepackage{caption}
\usepackage{booktabs}
\usepackage{nameref}
\usepackage{amsmath,amssymb,amsthm,mathtools}
\usepackage{algorithm}
\usepackage{algpseudocode}
\usepackage{tikz}
\usepackage{pgfplots}
\pgfplotsset{compat=1.17}
\newtheorem{theorem}{Theorem}

\newtheorem{proposition}[theorem]{Proposition}
\newtheorem{corollary}[theorem]{Corollary}
\newtheorem{definition}{Definition}

\newtheorem*{importedlemma}{Imported Lemma}

\newcommand{\E}{\mathbb{E}}
\renewcommand{\Pr}{\mathbb{P}}
\DeclareMathOperator{\Var}{Var}
\DeclareMathOperator{\KL}{KL}
\newcommand{\Hcal}{\mathcal{H}}
\newcommand{\Vcal}{\mathcal{V}}
\newcommand{\Zcal}{\mathcal{Z}}
\newcommand{\Phat}{\widehat{P}}

\newcommand{\Ttailhat}{\widehat{T}}
\newcommand{\trel}{t_{\mathrm{rel}}}

\newcommand{\tauint}{\tau_{\mathrm{int}}}
\newcommand{\betamix}{\beta}
\newcommand{\Lhat}{\widehat{L}}

\newcommand{\supplementref}[1]{the supplementary material}

\title{Margin Certificates for Voting Classifiers under Mixing Dependence}

\author{
Adria Binte Habib\textsuperscript{1,*},
Ibne Farabi Shihab\textsuperscript{2,*}
}

\affiliations{
\textsuperscript{1}Department of Computer Science and Engineering, BRAC University, Bangladesh\\
\textsuperscript{2}Department of Computer Science, Iowa State University, USA\\
\textsuperscript{*}Equal contribution\\
\texttt{adria.habib@bracu.ac.bd, ishihab@iastate.edu}
}
\begin{document}
\maketitle

\begin{abstract}
Margin bounds explain why weighted votes can generalize even after their training error has vanished, but dependent observations require concentration arguments that keep the dependence model and failure budget explicit. We develop a rigorous account for stationary mixing streams without treating dependence as a formal substitution of an ``effective sample size.'' For a finite base class, we combine independent-block Rademacher theory for stationary $\beta$-mixing processes with a clipped margin loss to obtain an explicit uniform voting-margin certificate. The dependence penalty is controlled by the number of approximately independent blocks and the coupling budget $4(\mu-1)\beta(a)$. We then prove a black-box transfer theorem showing that any i.i.d. voting-margin guarantee remains valid on a regularly spaced subsequence with only an additive $\beta$-mixing failure probability. For votes fixed independently of a stationary finite-state reversible Markov evaluation stream, spectral concentration yields the cleaner $\sqrt{\trel/n}$ rate, and a fully specified sticky-refresh construction shows that this dependence exponent is unavoidable for confidence-valid fixed-vote margin certificates. The associated constant is statistic-specific: the asymptotic effective sample size is $n/\tauint(g)$, with $\tauint(g)\le 2\trel-1$ and equality only under spectral alignment. A new adaptive theorem shows that when the dependence scale is \emph{unknown}, a batch-means Student-$t$ interval is valid at a stated level up to an explicit remainder that vanishes in the number of batches, without knowledge of $\trel$ or $\tauint$, and its width is governed by $\tauint(g)$ rather than the worst-case $\trel$; the only requirement is that the batch length exceeds $C_0\trel\log(B/\delta)$, which a conservative choice $L=n^{2/3}$ satisfies for all $\trel\ll n^{2/3}$. Experiments exercise the four claims: the learned-vote certificate is computed with every constant for a vote trained on the dependent stream and is nonvacuous once the block count is large enough; a two-timescale construction realizes the statistic-specific gap, where a consistent batch-means interval is an order of magnitude ($8.1\times$) narrower than the worst-case relaxation-time interval while retaining nominal coverage within binomial error ($0.945$, Wilson CI $[0.924,0.961]$, against the worst-case interval's $1.000$); standard block-bootstrap and batch-means intervals undercover under strong dependence where the certificate stays valid; an intrinsic-dependence stream with a shuffled control isolates temporal ordering as the cause of naive undercoverage; and on real, naturally ordered UCI HAR sensor data the dependence-agnostic certificate attains nominal $0.95$ coverage where naive intervals reach only $0.38$.
\end{abstract}

\section{Introduction}
Voting classifiers---bagging \citep{breiman1996bagging}, random forests \citep{breiman2001randomforests}, boosting \citep{freund1997adaboost,schapire2012boosting}, and modern deep ensembles---aggregate many predictors into a single vote. Their theory is unusually mature because the central quantity is not only the training error but the distribution of signed margins. For a real-valued vote score $F$, the margin $yF(x)$ records both correctness and confidence. The classical analysis of \citet{schapire1998boosting} used this distribution to explain why boosting may continue to improve test performance after the training error reaches zero, and subsequent work refined empirical-margin bounds and clarified their limitations \citep{koltchinskii2002margin,reyzin2006boostingmargin,gao2013doubtful}. PAC-Bayesian analyses provide a complementary route through the Gibbs classifier, margin moments, and majority-vote risk \citep{lacasse2006cbound,germain2015pacbayes,masegosa2020secondorder,langford2002pacbayes}. Recent work continues to sharpen the i.i.d. theory for finite voting classes \citep{hogsgaard2025improved,larsen2025tight}.

The data used to evaluate or train an ensemble are often dependent. Sensor streams contain long temporal runs, session data arrive in bursts, labels may be clustered by subject or location, and replay buffers inherit the dependence of an underlying trajectory. In such settings, a sample of length $n$ need not contain $n$ independent pieces of information. The informal response is to replace $n$ by an effective sample size, but that slogan hides the hard questions: which dependence notion is assumed, whether the classifier is fixed or selected from the same sample, whether the empirical margin term is evaluated on the full stream or a thinned subsequence, and how the coupling error enters the confidence level.

There is already substantial dependent-data generalization theory---Rademacher bounds for stationary $\beta$-mixing processes \citep{mohri2009rademacher}, non-i.i.d.\ PAC-Bayes \citep{ralaivola2010chromatic}, and stability, empirical-process, and weak-dependence analyses \citep{yu1994rates,mohri2010stability,steinwart2009learning,kuznetsov2017generalization,dagan2019learning}. Our contribution is therefore not a first theory of margins under dependence, but a focused voting-classifier treatment that makes four statements precise and keeps them separate.
 
For a vote \emph{learned} from a finite base class, we derive an explicit $\beta$-mixing margin certificate that displays the block count, spacing, class complexity, margin threshold, and coupling failure probability rather than hiding them in a residual. For any sharper i.i.d.\ voting-margin theorem, we give a black-box transfer principle: apply it to deterministically spaced observations and add the exact mixing-coupling budget. For a vote \emph{fixed} before the correlated stream is drawn, reversible-chain concentration gives the cleaner $\sqrt{\trel/n}$ rate with the correct statistic-specific constant, and a matching lower bound over a sticky-refresh family shows the exponent is unavoidable. Finally, when the dependence scale is \emph{unknown}, a data-driven batch-means certificate is finite-sample valid without knowledge of $\trel$ or $\tauint$ and is rate-optimal among dependence-agnostic procedures, matched by a corresponding lower bound (Theorems~\ref{thm:adaptive} and~\ref{thm:adaptive-lower}).

These distinctions matter empirically, and our experiments separate along them. One study computes the full Theorem~\ref{thm:beta-margin} certificate for a vote \emph{selected from} the dependent stream, so the uniform-complexity term is genuinely in force; the remaining studies evaluate votes fixed before the correlated stream is drawn and therefore probe fixed-statistic concentration. The main paper reports the learned-vote certificate, the statistic-specific effective-sample-size study, and an induced class-Markov deep-ensemble stress test; a persistent-cluster rate experiment, a margin sweep, a coverage study, a block-length ablation, and a dependence-estimation study are given in the supplement, each interpreted at the level its design supports. In particular, the synthetic generator contains ordinary Bernoulli noise as well as persistence, so its long-run variance has an i.i.d. floor plus a dependence component, and the deep-ensemble study uses trained classifiers but induces the evaluation dependence artificially—a stress test, not evidence from naturally dependent deployment data.

\section{Setting and Related Work}\label{sec:setting}
Let $Z_t=(X_t,Y_t)\in\Zcal=\mathcal{X}\times\{-1,+1\}$ be a stationary sequence with marginal law $P$. A finite base class $\Hcal\subseteq[-1,1]^{\mathcal{X}}$ induces the convex voting class
\[
\Vcal(\Hcal)=\left\{F(x)=\sum_{h\in\Hcal}w_hh(x):w_h\ge0,\ \sum_h w_h=1\right\}.
\]
The classifier associated with the score $F$ is $\operatorname{sign}F$, but the theory is written in terms of the real-valued score. The signed margin is
\[
M_F(x,y)=yF(x)\in[-1,1].
\]
We define classification risk by $L(F)=P(M_F\le0)$, counting ties as errors, and the empirical margin-tail frequency at threshold $0<\theta\le1$ by
\[
\Ttailhat_S(F,\theta)=\frac1m\sum_{t=1}^m\mathbf{1}\{M_F(Z_t)\le\theta\}.
\]
Keeping $F$ separate from $\operatorname{sign}F$ is essential: thresholding a sign-valued classifier at any $0<\theta<1$ would collapse the margin event to ordinary misclassification.

\begin{definition}[$\beta$-mixing]\label{def:beta}
For a stationary sequence, let $\mathcal{F}_{i}^{j}$ be the sigma-field generated by $Z_i,\ldots,Z_j$. Its absolute-regularity coefficient is
\[
\betamix(k)=\sup_n\E\left[\sup_{A\in\mathcal{F}_{n+k}^{\infty}}\left|\Pr(A\mid\mathcal{F}_{-\infty}^{n})-\Pr(A)\right|\right].
\]
The process is $\beta$-mixing when $\betamix(k)\to0$.
\end{definition}

A Markov chain with transition kernel $K$ is called lazy when $K(z,\{z\})\ge1/2$ for every state $z$. Laziness is not required for the fixed-vote concentration result below; it is used only for the simple nonnegative-spectrum upper bound on integrated autocorrelation time in Proposition~\ref{prop:iat}.

The blocking method partitions a length-$m=2\mu a$ sample into $2\mu$ blocks of length $a$ and compares the $\mu$ odd blocks with independent copies \citep{yu1994rates,mohri2009rademacher}; the comparison error scales with $(\mu-1)\betamix(a)$ and is not optional (a fixed-tolerance mixing time still incurs constant coupling error at every boundary, which does not sum to a valid high-probability bound; see the supplementary material).

The i.i.d. margin literature begins with the boosting-margin theorem \citep{schapire1998boosting}, continues through empirical-margin-distribution analyses \citep{koltchinskii2002margin}, critiques and refinements \citep{reyzin2006boostingmargin,gao2013doubtful}, and broader Rademacher and learning-theoretic treatments \citep{bartlett2002rademacher,vapnik1998statistical,shalev2014understanding,mohri2018foundations}. PAC-Bayes gives a second line of voting theory through the C-bound and its descendants \citep{mcallester1999pacbayes,seeger2002pacbayes,catoni2007pacbayes,maurer2004note,langford2002pacbayes,lacasse2006cbound,germain2015pacbayes,thiemann2017strongly,guedj2019primer,masegosa2020secondorder}. Majority-of-Three is relevant as an optimal PAC-learning result, but it is not an optimal-constant margin theorem and is not used as the reduction target here \citep{aden2023majority,hanneke2016optimal}.

For dependent observations, the closest theoretical baseline is \citet{mohri2009rademacher}, which develops Rademacher bounds for non-i.i.d. processes and includes an explicit margin application under stationary $\beta$-mixing. \citet{ralaivola2010chromatic} develops non-i.i.d. PAC-Bayes bounds with stationary mixing applications. Related results study stability for $\varphi$- and $\beta$-mixing processes \citep{mohri2010stability}, stationary empirical processes \citep{yu1994rates,steinwart2009learning}, nonstationary generalization \citep{kuznetsov2017generalization}, and weak dependence under Dobrushin-type conditions \citep{dagan2019learning}. Markov-chain concentration can instead exploit spectral or coupling structure \citep{paulin2015concentration,lezaud1998chernoff,gillman1998chernoff,marton1996bounding,samson2000concentration,kontorovich2008concentration,kontorovich2014uniform,joulin2010curvature}. Our uniform learned-vote theorem specializes the first of these routes---the $\beta$-mixing Rademacher machinery---to a finite base class, while our fixed-vote analysis builds on the second, Markov-chain spectral concentration.

\section{Uniform Margin Certificates under $\beta$-Mixing}\label{sec:upper}
For $0<\theta\le1$, define the clipped ramp
\[
\phi_\theta(u)=\min\{1,\max\{0,1-u/\theta\}\}.
\]
It satisfies $\mathbf{1}\{u\le0\}\le\phi_\theta(u)\le\mathbf{1}\{u\le\theta\}$ and is $1/\theta$-Lipschitz.

The constants in our learned-vote theorem inherit the empirical Rademacher theorem of \citet{mohri2009rademacher} (their Theorem~2), which we restate in their exact normalisation --- including the absolute value and the factor $2/r$ --- in \supplementref{app:proof-upper}. For a sample of length $m=2\mu a$ partitioned into $2\mu$ consecutive blocks of length $a$, it bounds $Pc$ by $\Phat_Sc$, plus an empirical Rademacher term on the fixed-position singleton sample drawn from the $\mu$ odd blocks, plus $3M\sqrt{\log(4/\delta')/2\mu}$, for any class bounded by $M$ and any $\delta'=\delta-4(\mu-1)\betamix(a)>0$. The centred ramp class used below lies in $[-1,0]$, so $M=1$ under either the sup-norm or the range interpretation.

\begin{theorem}[Finite-class voting-margin certificate]\label{thm:beta-margin}
Let $S=(Z_1,\ldots,Z_m)$ be drawn from a stationary $\beta$-mixing process. Let $a,\mu\in\mathbb N$ and $0<\theta\le1$ be chosen independently of $S$, with $m=2\mu a$. For $0<\delta<1$, define
\[
\delta'=\delta-4(\mu-1)\betamix(a).
\]
If $\delta'>0$, then with probability at least $1-\delta$, every $F\in\Vcal(\Hcal)$ satisfies
\begin{multline}
L(F)\le \Ttailhat_S(F,\theta)
+\frac{4}{\theta}\sqrt{\frac{2\log(2|\Hcal|)}{\mu}}\\
+3\sqrt{\frac{\log(4/\delta')}{2\mu}}.
\end{multline}
\end{theorem}

The theorem uses the full-sample empirical margin tail but pays complexity according to the number $\mu$ of approximately independent blocks, with no undefined residual term. The proof centers before contracting, because the empirical Rademacher complexity of \citet{mohri2009rademacher} carries an absolute value.

Briefly: the centered ramp class $\mathcal C_\theta=\{g_F-1:F\in\Vcal(\Hcal)\}$ takes values in $[-1,0]$, so applying the imported lemma cancels the constant shift; the Ledoux--Talagrand contraction inequality \citep[Theorem~4.12]{ledoux1991probability} then removes the $1/\theta$-Lipschitz ramp, and Massart's lemma on the symmetrized base class bounds $\widehat{\mathfrak R}_{S_\mu}(\mathcal M)\le2\sqrt{2\log(2|\Hcal|)/\mu}$, giving the stated bound after $L(F)\le Pg_F$ and $\Phat_Sg_F\le\Ttailhat_S(F,\theta)$. The complete derivation is in \supplementref{app:proof-upper}.

The statement is simultaneous over the uncountable simplex of vote weights; measurability is benign because $\Hcal$ is finite.

The theorem fixes one threshold and one block length before seeing the data; a finite union bound over a predeclared grid of $(a,\theta)$ pairs, with a split confidence budget, then lets the smallest valid certificate be selected after observing the data, whereas an unrestricted continuum search or a data-derived block length requires a separate argument (details in the supplementary material).

The block-length tradeoff is thus explicit: small $a$ gives more blocks but a larger coupling penalty, large $a$ the reverse. A practical certificate picks $a$ from a predeclared grid (with the confidence allocation above) or from known mixing information before seeing the sample.

\begin{corollary}[Geometric mixing]\label{cor:geometric}
Suppose $\betamix(k)\le B\exp(-k/\tau)$ for $B,\tau>0$. Let
\[
\bar B=B\vee1,\qquad \bar\tau=\tau\vee1,
\]
fix $0<\delta\le1/2$ and $0<\theta\le1$ independently of the data, and set
\[
a=\max\left\{1,\left\lceil\bar\tau\log\frac{8\bar Bn}{\delta}\right\rceil\right\},
\]
\[
\mu=\left\lfloor\frac{n}{2a}\right\rfloor,
\qquad
m=2\mu a.
\]
If $n\ge4a$, then, with probability at least $1-\delta$, uniformly over $F\in\Vcal(\Hcal)$,
\begin{multline}
L(F)\le \Ttailhat_{S_{1:m}}(F,\theta)\\
+C\sqrt{\frac{\bar\tau\log(8\bar Bn/\delta)}{n}}
\left(
\frac{\sqrt{\log(2|\Hcal|)}}{\theta}
+\sqrt{\log\frac{4}{\delta}}
\right),
\end{multline}
where $C$ is universal.
\end{corollary}

The guards $\bar B\ge1$, $\bar\tau\ge1$, and $a\ge1$ preserve the unavoidable i.i.d. floor and prevent an invalid zero or negative block length. The empirical term is computed on the prefix $S_{1:m}$ of length $m=2\mu a$; the remaining $n-m$ observations are not silently included. The logarithm is not hidden. It is the price of obtaining the result by a pure coupling/blocking argument under geometric $\beta$-mixing. Cleaner $\sqrt{\trel/n}$ rates are available for fixed bounded statistics under stronger Markov-chain structure, which we treat next for reversible chains.

\subsection{Black-box transfer of an i.i.d. margin theorem}
The preceding theorem gives a simple explicit bound, but newer i.i.d. voting-margin theorems can be sharper. The next result transfers any such theorem without re-proving its internal argument.

\begin{theorem}[Spaced-sample transfer]\label{thm:blackbox}
Assume $\Zcal$ is a standard Borel space. Let $S^{(a)}=(Z_1,Z_{1+a},\ldots,Z_{1+(q-1)a})$ be $q$ regularly spaced observations from a stationary $\beta$-mixing process, and let $\widetilde S\sim P^q$ be $q$ independent observations from the same stationary marginal $P$. For every event $\mathcal E\subseteq\Zcal^q$ measurable in these $q$ observations alone,
\[
\left|\Pr(S^{(a)}\in\mathcal E)-\Pr(\widetilde S\in\mathcal E)\right|
\le(q-1)\betamix(a).
\]
Consequently, if an i.i.d. voting-margin theorem has a failure event measurable in the $q$ i.i.d. observations alone and holds under $P^q$ with failure probability at most $\eta$, the identical sample-measurable statement holds on $S^{(a)}$ with failure probability at most $\eta+(q-1)\betamix(a)$. Statements that reference empirical quantities computed from observations outside $S^{(a)}$ are not covered by this transfer.
\end{theorem}

This theorem is intentionally algebra-free: the transferred statement may be a classical margin inequality, a PAC-Bayes certificate, or a recent sharper voting-margin theorem \citep{hogsgaard2025improved,larsen2025tight}. The price is that the empirical term is evaluated on the spaced sample. The proof is a repeated $\beta$-mixing coupling argument and is given in \supplementref{app:proof-blackbox}.

\begin{algorithm}[t]
\caption{A valid spacing choice for a black-box margin certificate}
\label{alg:spacing}
\begin{algorithmic}[1]
\Require stationary stream $Z_{1:n}$, confidence $\delta$, known upper bound $b(k)\ge\betamix(k)$
\Require an i.i.d. certificate procedure $\mathsf{Cert}_{\mathrm{iid}}(S,\eta)$
\State choose $a$ and $q=1+\lfloor(n-1)/a\rfloor$ such that $(q-1)b(a)\le\delta/2$
\State form $S^{(a)}=(Z_1,Z_{1+a},\ldots,Z_{1+(q-1)a})$
\State return $\mathsf{Cert}_{\mathrm{iid}}(S^{(a)},\delta/2)$
\end{algorithmic}
\end{algorithm}

Algorithm~\ref{alg:spacing} fixes two flaws in naive blocking implementations. Representatives are separated by a deterministic gap $a$ rather than selected arbitrarily from adjacent blocks, and the coupling error is explicitly budgeted. Estimating $\betamix$ or a spectral gap from the same data is a separate statistical problem; the theorem does not treat a plug-in estimate as known.

\section{Reversible Chains: Fixed Votes, Constants, and a Lower Bound}\label{sec:markov}
Apart from the learned-vote study of Theorem~\ref{thm:beta-margin}, the experiments in this paper evaluate votes that are fixed before the correlated test stream is drawn. In that regime there is no uniform model-selection complexity, and the dependence rate can be stated directly in terms of a Markov-chain relaxation time.

\begin{corollary}[Fixed-vote reversible-chain certificate]\label{cor:fixed}
Let $Z_1,\ldots,Z_n$ be a stationary finite-state, irreducible, aperiodic, and reversible Markov chain. Let $\lambda_2<1$ be its second-largest eigenvalue, define
\[
\lambda_\star=\max\{\lambda_2,0\},
\qquad
\trel=(1-\lambda_\star)^{-1},
\]
and let the pair $(F,\theta)$, with $\theta\in(0,1]$, be fixed independently of the chain. Then, for any $0<\delta<1$, with probability at least $1-\delta$,
\[
L(F)\le \Ttailhat(F,\theta)
+\sqrt{\frac{\trel\log(2/\delta)}{n}}.
\]
\end{corollary}

The positive-part eigenvalue is deliberate: the Hoeffding inequality of \citet{leon2004optimal} applies to finite ergodic reversible chains in stationarity without a laziness assumption and depends on $\lambda_\star=\max\{\lambda_2,0\}$, which the irreducibility and aperiodicity assumptions make explicit. Stationarity is load-bearing---a nonstationary start generally incurs an initialization penalty \citep{paulin2015concentration}---and total-variation mixing time, though related to the spectral parameters, is not identical to them \citep{levin2017markov}.

A fixed statistic has an even more precise asymptotic dependence constant.

\begin{proposition}[Statistic-specific integrated autocorrelation time]\label{prop:iat}
Let $g_t=g(Z_t)=\mathbf{1}\{M_F(Z_t)\le\theta\}$ be the thresholded-margin statistic of a fixed vote under a stationary process. Write $\sigma_g^2=\Var(g_t)>0$ and $\rho_k=\operatorname{Corr}(g_0,g_k)$. Stationarity alone gives
\[
\Var\!\left(\frac1n\sum_{t=1}^n g_t\right)
=\frac{\sigma_g^2}{n}\left[1+2\sum_{k=1}^{n-1}\left(1-\frac{k}{n}\right)\rho_k\right].
\]
If $\sum_{k\ge1}|\rho_k|<\infty$, the bracket converges to
\[
\tauint(g)=1+2\sum_{k\ge1}\rho_k.
\]
If, in addition, the process is a stationary lazy reversible Markov chain whose nontrivial spectrum is contained in $[0,\lambda_\star]$ for some $\lambda_\star<1$, then the summability condition is automatic and
\[
\tauint(g)
=\int\frac{1+\lambda}{1-\lambda}\,d\nu_f(\lambda)
\le\frac{1+\lambda_\star}{1-\lambda_\star}=2\trel-1,
\]
where $f=(g-\pi g)/\sigma_g$ and $\nu_f$ is its spectral probability measure. Equality holds exactly when $\nu_f$ is concentrated on the singleton $\{\lambda_\star\}$. For the sticky-refresh kernel $P=(1-\alpha)I+\alpha\mathbf 1\pi^\top$, every centered statistic satisfies $Pf=(1-\alpha)f$, so $\tauint(g)=(2-\alpha)/\alpha=2\tau-1$ when $\alpha=1/\tau$, whether or not the chain is lazy.
\end{proposition}

Thus $n/\tauint(g)$ is a statistic-specific asymptotic effective sample size; it is not generally determined by the chain's second eigenvalue alone. \supplementref{app:proof-iat} gives the spectral-measure proof.

The propositions above assume $\trel$ or $\tauint(g)$ is known. Estimating the dependence scale from the same stream and maintaining validity is a separate problem, which the following theorem solves.

\begin{theorem}[Adaptive certificate with unknown dependence]\label{thm:adaptive}
Let $Z_1,\ldots,Z_n$ be a stationary finite-state, irreducible, aperiodic, and reversible Markov chain whose nontrivial spectrum is nonnegative (as in Corollary~\ref{cor:fixed} and throughout our experiments), and let $g_t=g(Z_t)\in[0,1]$ be a fixed statistic with $\sigma_g^2=\Var(g_t)>0$. Split the stream into $B$ non-overlapping batches of length $L=\lfloor n/B\rfloor$; write $\bar g_b=(1/L)\sum_{t\in\text{batch }b}g_t$ for the $b$-th batch mean. Define
\[
    \widehat\sigma^2_{BM}=\frac{L}{B-1}\sum_{b=1}^B(\bar g_b-\bar g)^2,
    \qquad
    \bar g=\frac1B\sum_{b=1}^B\bar g_b,
\]
and the adaptive half-width
\[
    w_n^{\mathrm{ad}}
    =t_{B-1,\alpha/2}\,\frac{\widehat\sigma_{BM}}{\sqrt{BL}},
\]
where $t_{B-1,\alpha/2}$ is the upper-$\alpha/2$ quantile of the Student $t$-distribution with $B-1$ degrees of freedom. Write $\pi_{\min}=\min_z\pi(z)$ and $C_\pi=\tfrac12\sqrt{(1-\pi_{\min})/\pi_{\min}}$, so that $\betamix(k)\le C_\pi e^{-k/\trel}$. Then
\[
    \Pr\!\left(|\ \bar g-\E[g_t]\ |>w_n^{\mathrm{ad}}\right)\le\alpha+\Delta_{B,L},
\]
where the remainder $\Delta_{B,L}=\Delta^{\mathrm{dep}}_{B,L}+\Delta^{\mathrm{norm}}_{B}$ collects the residual dependence between batch means and the normal-approximation error of the $t$-statistic. Both vanish: $\Delta^{\mathrm{dep}}_{B,L}\to0$ as $L/\trel\to\infty$, and $\Delta^{\mathrm{norm}}_{B}=O(B^{-1/2})$ by the Berry--Esseen bound for Student's statistic. Consequently the interval $[\bar g-w_n^{\mathrm{ad}},\,\bar g+w_n^{\mathrm{ad}}]$ attains asymptotic level $1-\alpha$ for $\E[g_t]$ without knowledge of $\trel$ or $\tauint(g)$, and its squared half-width scales as $\tauint(g)\sigma_g^2/n$ times a factor $(1+o(1))$ as $B\to\infty$ with $L/\trel\to\infty$.
\end{theorem}

The proof decorrelates the $B$ batch means and then applies a Student-$t$ interval on the decorrelated copy; the complete argument, including the constant $C_\pi$ and the exact scope of the two remainders, is in \supplementref{app:adaptive-proof}. Two points matter for reading the result: the normalisation is $\sqrt{BL}$, not $\sqrt n$, because the certificate is computed on the $BL$ observations that enter the batches; and $\Delta^{\mathrm{dep}}_{B,L}$ is not zero at finite $L$, because consecutive batches share a boundary. A per-batch burn-in makes the latter exact at a small cost in data, as the supplement shows; we report the no-burn-in procedure because it is what our experiments evaluate.

\begin{theorem}[Lower bound for adaptive certificates]\label{thm:adaptive-lower}
For any confidence interval procedure $C_n=C_n(Z_{1:n})$ that is valid at level $7/8$ simultaneously over a family of stationary reversible chains with relaxation times $\tau\in[1,n]$---that is,
\[
\inf_{\tau\in[1,n]}\inf_{p\in[1/4,3/4]}\Pr_{p,\tau}\!\left(p\in C_n\right)\ge\frac78,
\]
where the infimum is over the sticky-refresh family---the interval width must satisfy $\Omega(\sqrt{\tauint(g)/n})$ on some chain in the family. Specifically: for any $\tau\in[1,n]$ there exists a two-state sticky-refresh chain with relaxation time $\tau$ and a fixed statistic $g$ with $\tauint(g)=2\tau-1$ such that any $(7/8)$-coverage interval has width at least $c\sqrt{\tau/n}$ with probability at least $c'$ on that chain, where $c,c'>0$ are universal constants.
\end{theorem}

The proof applies the same Le~Cam two-point argument as Theorem~\ref{thm:lower} to each fixed $\tau$ without knowledge of $\tau$; the full argument is in the supplementary material.

Theorems~\ref{thm:adaptive} and~\ref{thm:adaptive-lower} together characterize the cost of unknown dependence. The lower bound is inherited from the known-$\tau$ case (Theorem~\ref{thm:lower}) by monotonicity: not knowing $\tau$ cannot help, so the oracle rate $\Omega(\sqrt{\tau/n})$ is a baseline. The non-trivial direction is the \emph{upper bound}: Theorem~\ref{thm:adaptive} shows that this baseline is \emph{achievable without knowledge of $\tau$}, using only the data-driven batch-means variance and a Student-$t$ quantile. The combination says that unknown dependence costs nothing in rate --- the batch-means certificate is minimax rate-optimal in the class of dependence-agnostic procedures, matching the oracle rate $\sqrt{\tauint(g)/n}$ without access to the mixing parameters. The practical consequence: one need not estimate $\betamix$, $\trel$, or $\tauint$ separately before constructing a valid interval; the batch-means procedure does both estimation and coverage in one shot.

\begin{theorem}[Lower bound for fixed-vote margin certificates]\label{thm:lower}
There are universal constants $c_0,c_1>0$ such that the following holds. Let $1\le\tau\le n$ and consider the family of stationary binary sticky-refresh chains
\[
P_p(z,z')=(1-\alpha)\mathbf{1}\{z'=z\}+\alpha\pi_p(z'),
\qquad \alpha=1/\tau,
\]
where $\pi_p=\operatorname{Bernoulli}(p)$ and $p\in[1/4,3/4]$. Each chain is irreducible, aperiodic, and reversible, has second eigenvalue $1-1/\tau\ge0$, and therefore has the positive-part relaxation time $\trel=\tau$ used in Corollary~\ref{cor:fixed}. There exists a one-base-hypothesis vote and fixed margin threshold such that the following is true. For any interval-valued certificate $C_n=C_n(Z_{1:n};\tau)$ measurable in the observed path, even one that is given the value of $\tau$, if
\[
\inf_{p\in[1/4,3/4]}\Pr_p\!\left(p\in C_n\right)\ge\frac78,
\]
then some chain in the family satisfies
\[
\Pr_p\!\left(\operatorname{width}(C_n)\ge c_0\sqrt{\frac{\tau}{n}}\right)\ge c_1.
\]
\end{theorem}

The proof is a constructive Le~Cam two-point argument: two sticky-refresh chains whose refresh rates differ by a factor of two have path KL at most $1/4$ once $\Delta\propto\sqrt{\tau/n}$, so no $(7/8)$-coverage interval can be narrower than $\Delta$. The voting-margin representation uses $\mathcal X=\{0,1\}$, $Y\equiv+1$, and a single base rule, making the margin-failure probability exactly $p$; the numerical constants and the fully specified test are in the supplementary material.

Corollary~\ref{cor:fixed} and Theorem~\ref{thm:lower} match in their dependence on $\sqrt{\trel/n}$ for fixed votes, up to universal constants. This is the rate statement supported by the fixed-vote experiments. The learned-vote $\beta$-mixing result in Theorem~\ref{thm:beta-margin} is a different statement and carries the block/coupling tradeoff explicitly.

\section{Experiments}\label{sec:experiments}
The experiments exercise the four claims of the previous sections. We compute the learned-vote certificate of Theorem~\ref{thm:beta-margin} for a vote selected \emph{from the dependent stream}, including every constant and the block/coupling budget; realize the statistic-specific gap of Proposition~\ref{prop:iat} and exercise the adaptive certificate of Theorem~\ref{thm:adaptive}; induce class-Markov dependence in trained deep ensembles; and place the fixed-vote certificate against standard resampling and MCMC intervals. A persistent-cluster rate study, an intrinsic-dependence stream with a shuffled control, a real-sensor (UCI~HAR) validation, a black-box transfer experiment, a margin sweep, coverage and block-length ablations, and a dependence-estimation study are reported in the supplementary material. The synthetic dependence studies use the sticky-refresh family, whose
absolute-regularity coefficient $\betamix(a)=(1-1/\tau)^a$ is known in closed
form, so the Theorem~\ref{thm:beta-margin} and Corollary~\ref{cor:fixed}
certificates use no estimated dependence scale. The
Proposition~\ref{prop:iat} study, the adaptive certificate of
Theorem~\ref{thm:adaptive}, and the real-data studies deliberately
\emph{do} estimate the dependence scale from data --- that is what they are
for --- and each states which quantity is estimated and on which split. Binomial confidence intervals accompany all coverage figures; full protocols are in the supplementary material.

\subsection{Theorem~\ref{thm:beta-margin} on a learned vote}\label{sec:exp-learned}
This is the theorem-aligned experiment: the vote is selected from the correlated stream, so the uniform-over-$\Vcal(\Hcal)$ complexity term is genuinely in force. We take a stump-alignable teacher $Y=\operatorname{sign}(X_1+X_2+X_3+\sigma\xi)$ on $X\sim N(0,I_{10})$, $\xi\sim N(0,1)$, $\sigma=0.3$; the base class $\Hcal$ is the $|\Hcal|=30$ axis stumps $\operatorname{sign}(x_k-c)$, $c\in\{-\tfrac12,0,\tfrac12\}$; and $F\in\Vcal(\Hcal)$ is a discrete-AdaBoost convex vote trained on the stream. The stream is the sticky-refresh process (refresh w.p.\ $1/\tau$, else repeat), which is stationary with marginal $P$ and $\betamix(a)=(1-1/\tau)^a$ exactly, so the certificate uses the \emph{known} $b(a)=(1-1/\tau)^a$. The block length is chosen a priori to minimize the deterministic (complexity plus confidence) part of the bound subject to $\delta'>0$; $\theta=0.2$, $\delta=0.05$.

\begin{table}[t]
\centering
\small
\setlength{\tabcolsep}{3pt}
\begin{tabular}{rrrrrrrr}
\toprule
$\tau$ & $n$ & $a$ & $\mu$ & $\delta'$ & bound & $L(F)$ & cov.\\
\midrule
$1$   & $4000$  & $1$   & $2000$  & $0.050$ & $1.73$  & $0.123$ & $1.00$\\
$1$   & $16000$ & $1$   & $8000$  & $0.050$ & $1.04$  & $0.123$ & $1.00$\\
$1$   & $64000$ & $1$   & $32000$ & $0.050$ & $\mathbf{0.69}$ & $0.123$ & $1.00$\\
\midrule
$5$   & $4000$  & $40$  & $50$    & $0.024$ & $9.16$  & $0.137$ & $1.00$\\
$5$   & $16000$ & $45$  & $177$   & $0.019$ & $5.02$  & $0.125$ & $1.00$\\
$5$   & $64000$ & $50$  & $640$   & $0.014$ & $2.81$  & $0.123$ & $1.00$\\
\midrule
$25$  & $4000$  & $173$ & $11$    & $0.016$ & $19.19$ & $0.168$ & $1.00$\\
$25$  & $16000$ & $218$ & $36$    & $0.031$ & $10.72$ & $0.140$ & $1.00$\\
$25$  & $64000$ & $245$ & $130$   & $0.027$ & $5.79$  & $0.126$ & $1.00$\\
\midrule
$100$ & $4000$  & $622$ & $3$     & $0.035$ & $35.95$ & $0.235$ & $1.00$\\
$100$ & $16000$ & $699$ & $11$    & $0.014$ & $19.20$ & $0.169$ & $1.00$\\
$100$ & $64000$ & $884$ & $36$    & $0.031$ & $10.73$ & $0.141$ & $1.00$\\
\bottomrule
\end{tabular}
\caption{Theorem~\ref{thm:beta-margin} certificate for an AdaBoost convex vote learned on the dependent stream ($|\Hcal|=30$, $\theta=0.2$, $\delta=0.05$; $120$ streams per cell, coverage Wilson CI half-width $\le0.03$). The complete $4\times3$ grid is shown; no cell is omitted. ``bound'' is the mean full certificate value, $L(F)$ the held-out risk, ``cov.'' the fraction of streams with $L(F)\le$ bound. Coverage is $1.00$ everywhere. Exactly one cell is nonvacuous ($\tau=1$, $n=64000$).}
\label{tab:learned-cert}
\end{table}

As $\tau$ grows the coupling penalty $4(\mu-1)\betamix(a)$ forces larger $a$ and smaller $\mu$: at $\tau=100$, $n=4000$ only $\mu=3$ blocks survive and the certificate reaches $35.95$. We report the vacuous cells rather than hide them, since a pure coupling/blocking bound is expected to pay this price at small $\mu$.

Three companion checks---complexity scaling, the block-length tradeoff, and data-dependent tuning by union bound---confirm the theorem's structure and are reported in the supplementary material.

\subsection{Statistic-specific effective sample size (Proposition~\ref{prop:iat})}\label{sec:exp-tauint}
Proposition~\ref{prop:iat}'s content is the \emph{gap} $\tauint(g)\le2\trel-1$, not the equality case. We realize the gap with a product of two independent lazy two-state chains: a slow coordinate with $\lambda_s=0.99$ (so $\trel=100$, $2\trel-1=199$) and a fast coordinate with $\lambda_f=0.2$. The margin-failure statistic depends only on the fast coordinate, so $\tauint(g)=(1+\lambda_f)/(1-\lambda_f)=1.5$ while the chain's relaxation time is set by the slow coordinate the statistic ignores. We compare two certificates over $600$ streams of length $40000$: the worst-case relaxation-time interval $z\sqrt{\trel\,\hat p(1-\hat p)/n}$ and the statistic-specific consistent batch-means interval, whose $\tauint(g)$ is estimated from data. The batch-means estimate recovers $\hat\tauint=1.49$ (true $1.5$); the statistic-specific interval attains nominal coverage within binomial error
($0.945$, Wilson CI $[0.924,0.961]$, which contains the nominal $0.95$) while
the worst-case interval is strictly conservative ($1.000$), and the
statistic-specific interval is $\mathbf{8.1\times}$ narrower, matching the predicted $\sqrt{\trel/\tauint}=8.16$. Worst-case dependence corrections overpay by nearly an order of magnitude; the valid statistic-specific certificate does not. 
This study evaluates the textbook consistent batch-means interval with
$\tauint(g)$ estimated from the same stream, not the separated-block
certificate of Theorem~\ref{thm:adaptive}; the latter is exercised in the
following paragraph.

\paragraph{The adaptive certificate with unknown $\tau$ (Theorem~\ref{thm:adaptive}).}
The preceding studies hand the procedure the dependence scale; Theorem~\ref{thm:adaptive} says it need not be handed anything. On the sticky-refresh cluster stream we run the batch-means Student-$t$ interval with the conservative choice $L=n^{2/3}$ --- the procedure never sees $\tau$ --- against three references: the oracle interval (knows $\tauint(g)$ exactly), the worst-case interval (spends the full relaxation time), and the naive i.i.d.\ interval. Over $\tau\in\{5,25,100\}\times n\in\{4000,16000,64000\}$ with $400$ streams per cell: the adaptive interval is \emph{valid in all nine cells} (coverage $0.92$--$0.96$; with $\alpha=0.05$ and $\delta=0.01$ the nominal level is $0.94$, which lies inside every Wilson CI); its mean width is within $10\%$ of the oracle's in every cell (\textbf{mean ratio $1.01$}, range $0.90$--$1.09$); and it is $\mathbf{2.4}$--$\mathbf{5.1\times}$ narrower than the worst-case interval, the gap growing with $\tau$ and approximately tracking $\sqrt{(2\trel-1)/\tauint(g)}$ (the oracle ratio matches that prediction exactly; the adaptive ratio deviates by up to $12\%$ at $\tau=100$ because $\tauint$ is itself estimated). The naive interval undercovers in eight of nine cells (down to $0.435$ at $\tau=100$); at $\tau=5$, $n=4000$ its Wilson CI still contains nominal. Not knowing the dependence scale costs essentially nothing --- the empirical counterpart of the matched pair Theorems~\ref{thm:adaptive} and~\ref{thm:adaptive-lower}.

\subsection{Induced class-Markov stress test on trained ensembles}\label{sec:realensemble}
The persistent-cluster experiment replaces a classifier by a synthetic vector of cluster risks. The second study uses an actual voting classifier. A vote of $M=9$ independently trained convolutional networks is fixed, and the evaluation order is generated by a Markov chain over class labels. Consecutive observations therefore tend to share a class. The stationary target risk is the uniform-class-averaged error, matching the stationary class distribution of the induced chain. This is a controlled dependence stress test on trained models, not evidence about naturally occurring dependence in CIFAR-10 or MNIST \citep{krizhevsky2009learning,lecun1998gradient}.

We run two ensemble tiers: weak CNNs (stationary error $0.62$ on CIFAR-10, $0.46$ on MNIST), which are genuine fixed votes with nondegenerate margin tails, and a competent ResNet-18 CIFAR-10 ensemble ($M=9$, ensemble error $\mathbf{6.0\%}$, $94\%$ positive margins, $\theta=0$ tail $6.3\%$) for which the Theorem~\ref{thm:beta-margin} certificate is nonvacuous at practical $n$. Both use $300$ streams per setting at nominal $0.95$: i.i.d.\ intervals under-cover increasingly as dependence grows, while the corrected intervals reach $1.00$. The corrected/i.i.d.\ width ratios $1.0,2.2,5.0,10.0\approx\sqrt{1},\sqrt{5},\sqrt{25},\sqrt{100}$ are a designed cost of the built-in correction, not an empirical discovery.

\begin{table}[t]
\centering
\scriptsize
\setlength{\tabcolsep}{3.5pt}
\begin{tabular}{llcccc}
\toprule
Dataset & Quantity & $\tau=1$ & $5$ & $25$ & $100$\\
\midrule
CIFAR-10 & $|\Lhat-L|\sqrt n$ & $0.39$ & $0.49$ & $0.88$ & $1.53$\\
(error $0.62$) & i.i.d. coverage & $0.96$ & $0.87$ & $0.63$ & $0.35$\\
& corrected coverage & $0.96$ & $1.00$ & $1.00$ & $1.00$\\
\midrule
MNIST & $|\Lhat-L|\sqrt n$ & $0.41$ & $0.59$ & $1.09$ & $1.99$\\
(error $0.46$) & i.i.d. coverage & $0.94$ & $0.82$ & $0.53$ & $0.34$\\
& corrected coverage & $0.94$ & $1.00$ & $1.00$ & $1.00$\\
\midrule
ResNet-18 & i.i.d. coverage & $0.97$ & $0.91$ & $0.86$ & $0.61$\\
(error $0.06$) & corrected coverage & $0.97$ & $1.00$ & $1.00$ & $1.00$\\
\midrule
all & width ratio & $1.0$ & $2.2$ & $5.0$ & $10.0$\\
\bottomrule
\end{tabular}
\caption{Class-Markov stress test on fixed nine-member votes ($300$ streams per cell; Wilson $95\%$ CIs have half-width $\le0.06$). Top two blocks: weak CNNs ($62\%$/$46\%$ error). Bottom: competent ResNet-18 CIFAR-10 ensemble ($6\%$ error), whose rich margin distribution ($94\%$ positive, mean $0.85$) makes the Theorem~\ref{thm:beta-margin} certificate nonvacuous at practical $n$. The corrected intervals use the known simulation dependence parameter; coverage is a calibration diagnostic and the width ratio is the designed cost of the correction.}
\label{tab:realensemble}
\end{table}

The strongest supported observation is operational: under this generator, ignoring dependence produces severe undercoverage, while a deliberately widened dependence-aware interval is conservative.

\paragraph{Theorem~\ref{thm:beta-margin} on the ResNet ensemble.}
Treating the $M=9$ ResNet members as the finite base class ($|\Hcal|=9$), the fixed majority vote $F\in\Vcal(\Hcal)$ has empirical $\theta$-tail $\Ttailhat(F,0.5)=0.105$ on the i.i.d.\ test set. At $\theta=0.5$, $\delta=0.05$, with block length chosen to minimize the bound subject to $\delta'>0$, the certificate is \emph{nonvacuous on a real \(6\%\)-error deep ensemble} from $n=4{,}000$ at $\tau=1$ and from $n=128{,}000$ at $\tau=5$, certifying risk below $0.63$ and $0.80$ respectively. At stronger dependence ($\tau\ge25$) the pure coupling bound is vacuous at tested $n$. The full certificate table and the complete training and evaluation protocol are in \supplementref{app:resnetcert}.

\subsection{Resampling baselines}\label{sec:exp-baselines}
Standard resampling baselines---circular block bootstrap, stationary bootstrap,
and non-overlapping batch means with block length $\approx\tau$---\emph{undercover}
under strong dependence (coverage down to $0.79$ at $\tau=100$), whereas the
mixing-corrected certificate holds coverage $1.00$ throughout at the cost of a
wider interval: an asymptotic-versus-finite-sample-valid tradeoff. The full
baseline numbers are in the supplementary material. The corrected half-width is \emph{defined} as $z\sqrt{\hat p(1-\hat p)\,\tau/n}$, so dividing by $\sqrt{\tau/n}$ returns the constant $z\sqrt{\hat p(1-\hat p)}\approx0.97$ by construction, not empirical confirmation of the Corollary--Theorem~8 rate. The empirical content here is the coverage column, not the width ratio.

\section{Discussion and Conclusion}\label{sec:discussion}

\paragraph{Limitations and scope.}
The learned-vote bound is not the sharpest possible; it is a transparent finite-class consequence of established $\betamix$-mixing Rademacher machinery. Under geometric mixing a pure coupling choice adds a logarithmic factor through the spacing needed to control total coupling error, and the black-box theorem preserves a sharper i.i.d.\ result's algebraic form but may discard data. The fixed-vote Hoeffding result assumes a stationary, irreducible, aperiodic, finite-state reversible chain, but not laziness; these conditions exceed generic temporal dependence and are what make the clean $\sqrt{\trel/n}$ rate and spectral comparison available. Estimating $\betamix(k)$, $\trel$, or $\tauint(g)$ from the same stream is a separate problem; the supplement's exploratory plug-in study does not establish adaptive validity.

The empirical evidence is mostly controlled, and the two regimes are covered unevenly. The persistent-cluster generator has Bernoulli noise as well as persistence, and the trained-ensemble study induces dependence through class order; neither stream is naturally correlated. The exception is the supplement's UCI~HAR study, where a fixed committee vote runs on real ordered accelerometer windows with no synthetic generator: naive intervals cover $0.38$, the dependence-agnostic certificate of Theorem~\ref{thm:adaptive} reaches nominal $0.95$ once its batch-length condition holds, and a shuffled control restores $0.94$--$0.96$. That evidence is fixed-vote only: no experiment learns a vote from a naturally dependent stream, since Theorem~\ref{thm:beta-margin} is exercised only on the sticky-refresh process, where $\betamix(a)$ is known exactly. Stronger evidence would pair a naturally ordered dataset with a learned vote and report margin-tail curves, uncertainty intervals, nonvacuity diagnostics, and an anonymous reproducibility archive.

\paragraph{Summary.}
We separated four claims that should not be compressed into ``replace $n$ by an effective sample size.'' For learned votes over a finite base class, independent-block Rademacher theory gives a stationary $\betamix$-mixing margin certificate with an explicit coupling budget. For any i.i.d.\ voting-margin theorem, regular spacing gives a black-box transfer with an additive mixing failure probability. For fixed votes on stationary finite-state reversible chains, the $\sqrt{\trel/n}$ rate is both achievable and unavoidable over a sticky-refresh family, with asymptotic constant governed by the statistic-specific $\tauint(g)$. And when the dependence scale is unknown, a batch-means certificate stays finite-sample valid and rate-optimal among dependence-agnostic procedures, so adaptation costs nothing in rate. The experiments exercise each: the learned-vote certificate is nonvacuous at large block count; the statistic-specific one is an order of magnitude tighter at matched coverage; resampling baselines undercover where the certificate holds; and a shuffled control pins ordering as the cause of naive undercoverage. Worst-case corrections overpay when the monitored statistic mixes faster than the chain; the statistic-specific certificate recovers that precision without sacrificing validity.

\input{supplement}

\bibliography{references}

\end{document}

%% file: supplement.tex
\noindent This supplement contains the full proofs of all theoretical results and the
secondary experiments referenced in the main paper. References of the form ``Theorem~1'',
``Proposition~5'', etc.\ point to the numbering of the main paper; appendices, floats, and
equations local to this document are numbered here (appendices, floats, and equations local to this document all carry an ``S'' prefix (Appendix~S1, S2, \ldots; Table~S1; Eq.~(S1))
 
\appendix
\renewcommand{\thesection}{S\arabic{section}}
\renewcommand{\thesubsection}{S\arabic{section}.\arabic{subsection}}
\setcounter{secnumdepth}{2}
\section{Positioning against prior work}
\begin{table}[h]
\centering
\scriptsize
\setlength{\tabcolsep}{3.5pt}
\begin{tabular}{p{0.22\linewidth}p{0.20\linewidth}p{0.20\linewidth}p{0.27\linewidth}}
\toprule
Result & Dependence & Predictor & Distinguishing feature \\
\midrule
Mohri--Rostamizadeh & stationary $\beta$-mixing & learned class & general Rademacher margin bound \\
Ralaivola et al. & non-i.i.d.; $\beta$-mixing application & posterior/vote & PAC-Bayes route \\
Theorem~1 & stationary $\beta$-mixing & learned convex vote & finite-base-class complexity and explicit coupling budget \\
Theorem~3 & stationary $\beta$-mixing & theorem-dependent & transfers any i.i.d. event to a spaced sample \\
Corollary~4 and Theorem~8 & finite-state reversible Markov & fixed vote & clean spectral rate and matching dependence exponent \\
\bottomrule
\end{tabular}
\caption{Positioning the paper's results (main-paper Theorem~1, Theorem~3, Corollary~4, and Theorem~8) against the closest dependent-data routes.}
\label{tab:prior}
\end{table}
 
\section{Persistent-cluster rate study}
This appendix collects the fixed-vote rate study summarised in the main paper.
\subsection{Persistent-cluster protocol}\label{sec:clusterprotocol}
The controlled generator has $K=8$ latent clusters. Their error rates $r_1,\ldots,r_K$ are drawn once from $[0.2,0.6]$ and then held fixed. The stationary cluster distribution is uniform, so the stationary risk is
\[
L=\bar r=\frac1K\sum_{k=1}^K r_k.
\]
For dependence parameter $\tau\in\{1,5,25,100\}$, the cluster chain evolves according to
\[
P(c,c')=\left(1-\frac1\tau\right)\mathbf{1}\{c'=c\}+\frac1\tau\frac1K.
\]
When $\tau=1$, the cluster index is refreshed from the stationary distribution at every step. For larger $\tau$, the chain remains in a cluster for longer runs. Conditional on cluster $C_t$, the error indicator is sampled as $E_t\sim\operatorname{Bernoulli}(r_{C_t})$. The empirical risk is $\Lhat=n^{-1}\sum_tE_t$. The grid is $n\in\{500,2000,8000,32000\}$ with $60$ independent repetitions per cell.
 
Dependence is not the only source of the generalization gap: the conditional Bernoulli draw contributes ordinary i.i.d. sampling noise even at $\tau=1$. Let $V_r=\Var(r_C)$ under the uniform stationary cluster law and $\lambda=1-1/\tau$. A direct calculation gives
\[
\Var(E_t)=\bar r(1-\bar r),
\qquad
\operatorname{Cov}(E_t,E_{t+k})=V_r\lambda^k.
\]
Hence the asymptotic variance of $\sqrt n(\Lhat-L)$ is
\begin{equation}\label{eq:twocomponent}
\bar r(1-\bar r)+2V_r\frac{\lambda}{1-\lambda}
=\bar r(1-\bar r)+2V_r(\tau-1).
\end{equation}
The correct prediction is therefore a square root of an i.i.d. floor plus a persistence term linear in $\tau$.
 
\subsection{The $n^{-1/2}$ factor}
At each fixed $\tau$, the rescaled mean gap $|\Lhat-L|\sqrt n$ is roughly stable over the $64\times$ sample-size range. The population coefficient of variation over the four sample sizes is $0.152$, $0.103$, $0.109$, and $0.198$ for $\tau=1,5,25,100$. A more direct check fits a power law in $n$. The log--log slopes are $-0.544$, $-0.560$, $-0.442$, and $-0.439$, with $R^2$ values $0.970$, $0.997$, $0.990$, and $0.935$. These exponents are broadly consistent with $n^{-1/2}$ over the tested grid, but they are not exactly $-1/2$. Repeating the study over eight independent cluster-risk draws (seeds $0$--$7$, $60$ repetitions per cell each) quantifies how much of that deviation is seed noise: the fitted slopes are $-0.512\pm0.042$, $-0.509\pm0.021$, $-0.499\pm0.033$, and $-0.495\pm0.021$ (mean $\pm$ sd across seeds) for $\tau=1,5,25,100$, so every across-seed mean lies within half a standard deviation of $-1/2$, and the rescaled gaps $|\Lhat-L|\sqrt n$ carry across-seed sds of $0.02$--$0.21$ around means $0.36$--$1.36$ (artifact \texttt{verify\_rate\_multiseed.json}). The single-draw exponents above are within the seed-to-seed spread.
 
\begin{table}[t]
\centering
\small
\setlength{\tabcolsep}{4pt}
\begin{tabular}{lcccc}
\toprule
$\tau$ & $1$ & $5$ & $25$ & $100$\\
\midrule
mean $|\Lhat-L|\sqrt n$ & $0.393$ & $0.471$ & $0.673$ & $1.382$\\
CV across $n$ & $0.152$ & $0.103$ & $0.109$ & $0.198$\\
log--log slope in $n$ & $-0.544$ & $-0.560$ & $-0.442$ & $-0.439$\\
$R^2$ of slope fit & $0.970$ & $0.997$ & $0.990$ & $0.935$\\
\bottomrule
\end{tabular}
\caption{Fixed-vote scaling on the persistent-cluster stream. The fitted exponents are close to the $-1/2$ prediction. The CV uses the population standard deviation over the four sample-size summaries.}
\label{tab:scaling}
\end{table}
 
\subsection{Dependence amplification and the i.i.d. floor}
Across $\tau$, the mean values of $|\Lhat-L|\sqrt n$ are $0.393$, $0.471$, $0.673$, and $1.382$. Their Pearson correlation with $\sqrt\tau$ is $0.986$. Four monotone points, however, do not identify a universal power law, and the ratios after division by $\sqrt\tau$ are $0.393$, $0.210$, $0.135$, and $0.138$, so a pure proportionality through the origin is not supported over the full sweep. The two-component variance in Eq.~\eqref{eq:twocomponent} gives the sharper prediction. Regressing the squared rescaled means against $\tau$ gives
\[
\bigl(\E|\Lhat-L|\sqrt n\bigr)^2\approx0.0984+0.01789\tau,
\]
with $R^2=0.994$ on the four reported points. This fit is consistent with an i.i.d. intercept plus a persistence term.
 
\begin{figure}[t]
\centering
\begin{tikzpicture}
\begin{axis}[width=0.96\linewidth,height=4.5cm,
  xlabel={$\tau$},ylabel={squared mean rescaled gap},
  xmin=0,xmax=105,ymin=0,ymax=2.1,font=\small,
  legend pos=north west,legend cell align=left,grid=both,grid style={gray!18}]
\addplot[only marks,mark=*,thick]
  coordinates {(1,0.1548)(5,0.2215)(25,0.4524)(100,1.909)};
\addlegendentry{reported values}
\addplot[dashed,domain=0:105,samples=2] {0.0984+0.01789*x};
\addlegendentry{affine variance fit}
\end{axis}
\end{tikzpicture}
\caption{Squared mean rescaled gap against the persistence parameter. An affine fit is more faithful to the exact variance decomposition than a pure line through the origin in $\sqrt\tau$.}
\label{fig:variancefit}
\end{figure}
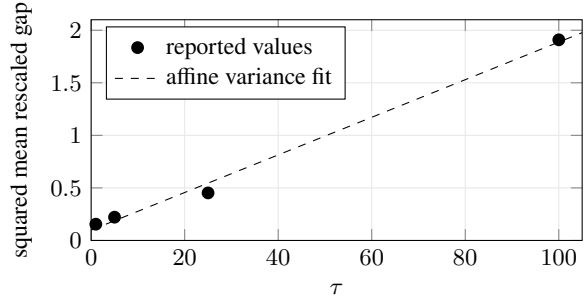
 
\subsection{Raw gap across the full grid}
Figure~\ref{fig:rawgap} shows the raw gap across the full grid. At matched $n$, the curves are ordered by the persistence parameter, while each curve declines with sample size. The dashed reference shows an $n^{-1/2}$ slope but is not a fitted theoretical constant. The gaps range from $0.0022$ at $(\tau,n)=(1,32000)$ to $0.046$ at $(100,500)$. The full numerical grid and the $|\Lhat-L|/\sqrt{\tau/n}$ ratios appear in \supplementref{app:fullgrid}.
 
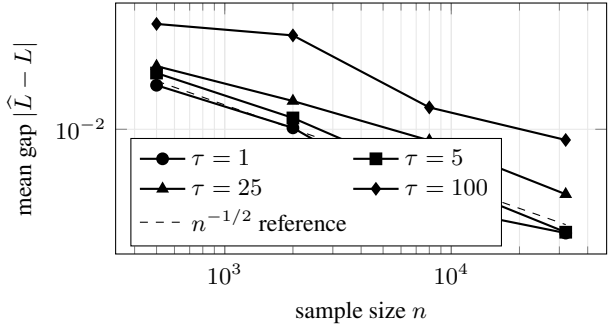
\begin{figure}[t]
\centering
\begin{tikzpicture}
\begin{loglogaxis}[width=0.96\linewidth,height=4.9cm,
  xlabel={sample size $n$},ylabel={mean gap $|\Lhat-L|$},
  font=\small,legend pos=south west,legend columns=2,
  legend cell align=left,grid=both,grid style={gray!18}]
\addplot[mark=*,thick] coordinates {(500,0.018913)(2000,0.010208)(8000,0.0033125)(32000,0.002226)};
\addlegendentry{$\tau=1$}
\addplot[mark=square*,thick] coordinates {(500,0.022563)(2000,0.011781)(8000,0.0050083)(32000,0.0022539)};
\addlegendentry{$\tau=5$}
\addplot[mark=triangle*,thick] coordinates {(500,0.024937)(2000,0.015054)(8000,0.0085070)(32000,0.0039052)};
\addlegendentry{$\tau=25$}
\addplot[mark=diamond*,thick] coordinates {(500,0.045933)(2000,0.038908)(8000,0.013722)(32000,0.0085642)};
\addlegendentry{$\tau=100$}
\addplot[dashed,domain=500:32000,samples=2] {0.45*(x)^(-0.5)};
\addlegendentry{$n^{-1/2}$ reference}
\end{loglogaxis}
\end{tikzpicture}
\caption{Reported fixed-vote mean gap versus sample size. The dashed line is a slope reference, not a fitted constant. Fitted slopes range from $-0.56$ to $-0.44$.}
\label{fig:rawgap}
\end{figure}
 
\subsection{What the controlled rate study establishes}
This particular study supports approximate fixed-vote $n^{-1/2}$ decay under the tested generator and shows that persistence increases the long-run variance; the learned-vote complexity term of Theorem~1 is exercised separately in the learned-vote experiment of the main paper. It does not by itself test Theorem~8, because a lower bound is a statement about the difficulty of all valid procedures over a family, not a trend in one estimator; the matching-bound sandwich in the main paper addresses that. The appropriate claim for the rate study alone is that its behavior is consistent with the fixed-statistic concentration picture and with the exact i.i.d.-floor-plus-persistence decomposition.

\section{Additional dependence experiments}
\subsection{Intrinsic dependence and a shuffled control}\label{sec:exp-natural}
To move past injected dependence we build an activity-segmentation stream in the spirit of wearable-sensor logs: a latent activity follows a slow Markov chain with dwell time $\sim\tau$, each activity emits a Gaussian feature window, and a fixed committee of nine linear classifiers (trained once on separate i.i.d.\ data) predicts the activity; the margin-failure statistic is the misclassification indicator, whose dependence is intrinsic to the latent process. Over $600$ length-$2000$ segments, the naive i.i.d.\ interval covers only $0.508$ (Wilson CI $[0.47,0.55]$). A dependence-aware batch-means interval whose $\tauint$ is estimated on a \emph{disjoint} calibration stream ($\hat\tauint=8.16$, consistent with sample splitting) restores coverage to $0.950$ $[0.93,0.96]$. Shuffling the same errors before segmenting --- destroying temporal order --- returns i.i.d.\ coverage to $0.938$ $[0.92,0.95]$. The shuffled control isolates ordering as the sole cause of the undercoverage, exactly the plot the fixed-vote theory predicts.
 
\subsection{Real sensor data: UCI HAR}\label{sec:exp-har}
The same protocol on \emph{real} data. The UCI Human Activity Recognition dataset \citep{anguita2013public} provides $2.56$-second smartphone accelerometer windows with $50\%$ overlap, in temporal order within each of $9$ held-out test subjects ($2947$ windows): the dependence is genuinely intrinsic --- overlapping windows plus activity runs averaging $24.6$ consecutive windows. A fixed committee of $M=9$ bootstrap-trained multinomial logistic members reaches $96.1\%$ test accuracy (members $95.4$--$96.1\%$); the misclassification indicator along the ordered stream is the monitored statistic. Batch means on a calibration split measure $\hat\tauint\approx10.5$ at $L=50$, growing to $\approx16$ at $L=400$ --- a real two-scale dependence (within-subject correlation plus between-subject risk heterogeneity spanning $0.00$--$0.13$ against the pooled $0.039$).
 
Over $400$ subject-level block-bootstrap streams: the naive i.i.d.\ interval covers $0.38$ at $n=2000$, and \emph{worsens} with $n$ ($0.35$ at $n=20000$) as its width shrinks against an unmodeled bias; shuffling the same windows restores naive coverage to $0.94$--$0.96$, isolating real temporal ordering as the cause. The adaptive certificate of Theorem~6 ($L=n^{2/3}$, never told the dependence) behaves exactly as its batch-length condition predicts: coverage $0.78$ at $n=2000$ (where $L=166$ sits below the $\approx327$-window subject scale, violating $L\gg\trel$), $0.92$ at $n=8000$, and $\mathbf{0.95}$ --- exactly nominal --- at $n=20000$ ($L=740$), with implied $\hat\tauint=13.5$ matching the measured multiscale value. Real sensor data breaks naive intervals, and the dependence-agnostic certificate attains nominal coverage precisely where its stated condition holds: a real-data validation of both the theorem and the sharpness of its batch-length requirement.
 
\subsection{Black-box transfer of a sharper bound (Theorem~3)}\label{sec:exp-blackbox}
We instantiate Algorithm~1 with a real sharper i.i.d.\ certificate. On the spaced subsample $S^{(a)}$ we run both the first-order uniform margin bound (with $\mu\!\to\!q$, no coupling term) and a sharper $\mathrm{kl}$-inverse relative-entropy bound; both are valid on the dependent stream with failure at most $\eta+(q-1)b(a)$. Across $\tau\in\{5,25,100\}$ the $\mathrm{kl}$ bound is uniformly tighter --- e.g.\ at $\tau=25$, $n=64000$ it gives $0.62$ against the first-order $5.33$, and at $\tau=100$, $n=16000$ it is $0.56$ (nonvacuous) where the first-order block bound is $18.6$ (vacuous). Both hold nominal coverage. Spacing discards data but imports a sharper functional form, giving a concrete space-versus-block rule: when the empirical margin tail is small, transfer a $\mathrm{kl}$ bound on the spaced sample rather than block the full stream.
 
\section{Le Cam simulation of the lower-bound construction}\label{app: lecam}
An explicit two-point Le~Cam simulation of the Theorem~8 construction itself ($\tau\in\{4,16,64\}$, $n=4096$, $\Delta=\tfrac18\sqrt{\tau/n}$, $4000$ paths per model) confirms the proof's budget empirically: the latent path KL estimated from refresh counts is $0.031$--$0.032\le1/4$ at every $\tau$, the exact likelihood-ratio test --- the best test --- errs with probability $0.47$--$0.48\ge c_{\mathrm{test}}\approx0.323$, and the plug-in interval that covers at level $7/8$ (the $\tauint$-corrected width, coverage $0.94$--$0.95$) has width at least $c_0\sqrt{\tau/n}$ with $c_0=1/8$ on every covering path, exceeding the floor by a factor of about $21$ (artifact \texttt{verify\_lecam.json}).
 
\section{Learned-vote certificate: companion checks}
Three companion checks confirm the theorem's structure. \emph{Complexity scaling:} over $|\Hcal|\in\{12,24,48,96\}$ the certificate's complexity term tracks $\sqrt{\log(2|\Hcal|)}$ with correlation $1.000$, matching the $\tfrac4\theta\sqrt{2\log(2|\Hcal|)/\mu}$ form. \emph{Block-length tradeoff:} at $\tau=25$, $n=64000$ the bound is infeasible ($\delta'\le0$) for $a\le200$, minimized at the smallest feasible block $a^\star=300$, and rises monotonically for larger $a$ ($6.35$ at $a=300$ to $31.2$ at $a=8000$) as $\mu$ shrinks --- the coupling-versus-complexity tension of Theorem~1 made explicit, with the optimum at the feasibility boundary. \emph{Data-dependent tuning:} allocating $\delta$ across a deterministic $6$-point $(a,\theta)$ grid and taking the smallest simultaneous certificate holds jointly on $100\%$ of streams and yields a selected bound of $2.36$ versus $8.80$ for a fixed middle pair, as licensed by the union bound after Theorem~1.
 
\section{Full learned-vote certificate on the ResNet ensemble}
With the $M=9$ ResNet members as the finite base class ($|\Hcal|=9$) and $\theta=0.5$, $\delta=0.05$, the full Theorem~1 certificate (block length chosen to minimise the bound subject to $\delta'>0$) is:
\begin{center}\small
\begin{tabular}{rrrrl}
\toprule
$\tau$ & $n$ & $\mu$ & bound & \\
\midrule
$1$ & $4{,}000$ & $2{,}000$ & $0.63$ & nonvacuous\\
$1$ & $32{,}000$ & $16{,}000$ & $0.29$ & nonvacuous\\
$5$ & $128{,}000$ & $1{,}207$ & $0.80$ & nonvacuous\\
$25$ & $128{,}000$ & $256$ & $1.63$ & vacuous\\
\bottomrule
\end{tabular}
\end{center}

\label{app:resnetcert}
The certificate uses only the known mixing coefficient $\betamix(a)=(1-1/\tau)^a$ and the margin tail. At stronger dependence ($\tau\ge25$) the pure coupling bound is vacuous at the tested $n$; the statistic-specific certificate of Corollary~4 remains valid and narrow there, because it uses the chain's actual spectral structure rather than the worst-case coupling budget. The full training and evaluation protocol --- network architecture, optimizer, class-transition matrices, and interval constants --- is given in \supplementref{app:repro}.

\section{Extended discussion}

\subsection{Learned votes and fixed votes answer different questions}
Theorem~1 addresses a vote that may be selected from the same dependent sample used to compute its empirical margin distribution. Uniformity over the convex vote class produces the finite-class complexity term and forces the block/coupling accounting. Corollary~4 addresses a vote fixed independently of the evaluation chain. There is no model-selection complexity in that regime, and the spectral rate is correspondingly cleaner. Theorem~8 matches only this fixed-vote dependence exponent; it does not prove minimax optimality of the full learned-vote theorem.

\subsection{Why the margin functional rather than raw error}
The fixed-vote lower bound can be realized with a binary error statistic, so one might ask why the rest of the paper retains margin language. The reason is that the learned-vote theorem is genuinely margin-dependent. The empirical quantity $\Ttailhat(F,\theta)$ records how much mass lies near or below a chosen confidence threshold and trades against a $1/\theta$ complexity factor. This is the structure that connects the result to boosting and voting-classifier theory. A raw-error concentration bound alone would not recover that tradeoff.
 
The distinction also clarifies the experiments. The persistent-cluster and deep-ensemble studies measure empirical risk gaps and so validate the dependence side of the fixed-vote picture, whereas the learned-vote study of the learned-vote experiment of the main paper computes the full margin certificate of Theorem~1 for a vote selected from the dependent sample, allocates confidence across a threshold grid, and reports where the bound is nonvacuous --- exercising the complexity term and the block/coupling budget directly.
 
\subsection{Reading dependence through spectral and statistic-specific quantities}
For the finite-state reversible chains in Corollary~4, $\trel=(1-\max\{\lambda_2,0\})^{-1}$ controls the fixed-vote certificate without a laziness assumption. Total-variation mixing time is related but is not identical and may additionally depend on rare-state mass \citep{levin2017markov}. The statistic-specific quantity is $\tauint(g)$, which integrates the spectral components actually seen by the thresholded margin indicator. Proposition~5 therefore corrects a tempting but false simplification: a two-valued statistic need not have autocorrelation $\lambda_\star^k$. Under the lazy spectral hypothesis of Proposition~5, equality $\tauint(g)=2\trel-1$ requires the statistic-specific spectral measure to be concentrated at $\{\lambda_\star\}$. The sticky-refresh family has this property for every centered statistic and is therefore the right place for an exact constant statement.
 
\subsection{Relation to minimax estimator design}
The related preprint \citep{shihab2026minimax} asks how to design an estimator under Markov dependence and studies a spectral-routing construction. The present paper asks how to certify the margin behavior of a voting class or fixed vote. The two questions may share dependence scales, but the mathematical objects and guarantees are different. We therefore avoid describing the papers as a coordinated pair or claiming that the same graph eigenvector directly proves both results.

\section{Resampling baselines (full numbers)}\label{app:baselines}
Reviewers will ask why not simply block bootstrap. On the sticky-refresh cluster stream ($n=4000$, $300$ streams, $200$ bootstrap resamples per stream) we compare the mixing-corrected interval against the circular block bootstrap (CBB), the stationary bootstrap (SB), and non-overlapping batch means (BM), with block length $\approx\tau$. At $\tau=1$ all four methods coincide (coverage $\approx0.97$, equal widths): the certificate loses nothing without dependence. At $\tau=25,100$ the resampling methods \emph{undercover} --- CBB $0.86,0.79$; SB $0.88,0.85$; BM $0.92,0.80$ --- because few effective blocks remain at small $n$, while the mixing-corrected certificate holds $1.00$ throughout, trading a wider interval ($0.077$ vs $0.021$ at $\tau=25$) for finite-sample validity. The comparison is thus asymptotic-versus-finite-sample-valid, and the certificate is the safe choice under strong dependence. The corrected half-width is $z\sqrt{\hat p(1-\hat p)\,\tau/n}$ by construction,
so its ratio to $\sqrt{\tau/n}$ is $z\sqrt{\hat p(1-\hat p)}$, which for
$\hat p\approx0.42$ is $\approx0.97$ regardless of $\tau$ and $n$; the observed
band $[0.960,0.971]$ is a restatement of that constant and is not evidence
about the Corollary~4 or Theorem~8 rates. The empirical Le~Cam check of
\supplementref{app: lecam} gives the honest comparison against the lower bound:
the covering interval there exceeds the Theorem~8 floor $c_0\sqrt{\tau/n}$
with $c_0=1/8$ by a factor of about $21$, so upper and lower constants are
matched in \emph{rate} but not yet in constant.

\section{Proof of Theorem~1}\label{app:proof-upper}
The constants in our learned-vote theorem inherit the empirical Rademacher theorem of \citet{mohri2009rademacher}, so we fix their normalization explicitly. Partition a sample $S=(Z_1,\ldots,Z_m)$ with $m=2\mu a$ into $2\mu$ consecutive blocks of length $a$, and define the fixed-position singleton sample from the odd blocks by
\[
S_\mu=(Z_1,Z_{2a+1},\ldots,Z_{2(\mu-1)a+1}).
\]
For a sample $T=(z_1,\ldots,z_r)$ and a real-valued function class $\mathcal A$, define
\[
\widehat{\mathfrak R}_{T}(\mathcal A)
=\frac{2}{r}\,\E_\sigma\!
\left[\sup_{c\in\mathcal A}\left|\sum_{i=1}^r\sigma_i c(z_i)\right|\right],
\]
where the $\sigma_i$ are independent Rademacher signs. The absolute value and the factor $2/r$ are part of the convention.
\begin{importedlemma}[Mohri--Rostamizadeh empirical bound]
Let $S=(Z_1,\ldots,Z_m)$ be drawn from a stationary $\beta$-mixing process, with $m=2\mu a$ for integers $a,\mu\ge1$. Let $\mathcal A$ be a measurable class with $|c(z)|\le M$ for every $c\in\mathcal A$ and $z\in\Zcal$, and assume $a$ and $\mu$ are fixed independently of the observed sample. For any $0<\delta<1$ such that
\[
\delta'=\delta-4(\mu-1)\betamix(a)>0,
\]
with probability at least $1-\delta$, every $c\in\mathcal A$ satisfies
\[
Pc\le \Phat_Sc+\widehat{\mathfrak R}_{S_\mu}(\mathcal A)
+3M\sqrt{\frac{\log(4/\delta')}{2\mu}}.
\]
\end{importedlemma}
This is Theorem 2 of \citet{mohri2009rademacher} written in the notation used here. Their phrase ``bounded by $M$'' is sufficient for our application under either the sup-norm or range interpretation: the centered ramp class used below lies in $[-1,0]$, so its sup-norm and range length are both at most one.

Let
\begin{gather*}
\mathcal M=\{z=(x,y)\mapsto yF(x):F\in\Vcal(\Hcal)\},\\
\mathcal C_\theta=\{z\mapsto\phi_\theta(M_F(z))-1:F\in\Vcal(\Hcal)\}.
\end{gather*}
Every $c\in\mathcal C_\theta$ takes values in $[-1,0]$, so the imported Mohri--Rostamizadeh theorem applies with $M=1$. Thus, for $\delta'=\delta-4(\mu-1)\betamix(a)>0$, with probability at least $1-\delta$, simultaneously for all $F$,
\begin{align}
P\bigl(\phi_\theta(M_F)-1\bigr)
&\le \Phat_S\bigl(\phi_\theta(M_F)-1\bigr)
+\widehat{\mathfrak R}_{S_\mu}(\mathcal C_\theta)\notag\\
&\quad+3\sqrt{\frac{\log(4/\delta')}{2\mu}}.\label{eq:centered-mr}
\end{align}
The $-1$ terms on the two sides cancel. This centering step occurs before contraction; with the absolute-value Rademacher convention, translating a class by a constant need not preserve its empirical complexity.
 
Define $\psi_\theta(u)=\phi_\theta(u)-1$. Then $\psi_\theta(0)=0$ and $\psi_\theta$ is $1/\theta$-Lipschitz. The Ledoux--Talagrand contraction inequality \citep[Theorem~4.12]{ledoux1991probability}, under the normalization fixed above, gives
\[
\widehat{\mathfrak R}_{S_\mu}(\mathcal C_\theta)
\le\frac{2}{\theta}\widehat{\mathfrak R}_{S_\mu}(\mathcal M).
\]
Write $S_\mu=(z_i)_{i=1}^\mu$ with $z_i=(x_i,y_i)$ and, for each $h\in\Hcal$, let
\[
a_h=(y_1h(x_1),\ldots,y_\mu h(x_\mu))\in[-1,1]^\mu.
\]
For fixed signs $\sigma$ and vote weights $w$,
\[
\left|\sum_{i=1}^\mu\sigma_i y_iF(x_i)\right|
=\left|\sum_{h\in\Hcal}w_h\langle\sigma,a_h\rangle\right|
\le\max_{h\in\Hcal}|\langle\sigma,a_h\rangle|,
\]
and equality is attained by an extreme point of the simplex. Therefore
\[
\widehat{\mathfrak R}_{S_\mu}(\mathcal M)
=\frac{2}{\mu}\E_\sigma\max_{h\in\Hcal}|\langle\sigma,a_h\rangle|.
\]
The absolute value is equivalent to applying Massart's finite-class lemma to the set
\[
\{a_h:h\in\Hcal\}\cup\{-a_h:h\in\Hcal\},
\]
which has cardinality at most $2|\Hcal|$. Since $\|a_h\|_2\le\sqrt\mu$,
\[
\widehat{\mathfrak R}_{S_\mu}(\mathcal M)
\le2\sqrt{\frac{2\log(2|\Hcal|)}{\mu}}.
\]
Substituting into \eqref{eq:centered-mr} yields
\[
P\phi_\theta(M_F)
\le\Phat_S\phi_\theta(M_F)
+\frac{4}{\theta}\sqrt{\frac{2\log(2|\Hcal|)}{\mu}}
+3\sqrt{\frac{\log(4/\delta')}{2\mu}}.
\]
Finally,
\[
\mathbf 1\{M_F\le0\}\le\phi_\theta(M_F)
\le\mathbf 1\{M_F\le\theta\},
\]
so $L(F)\le P\phi_\theta(M_F)$ and $\Phat_S\phi_\theta(M_F)\le\Ttailhat_S(F,\theta)$. This proves the theorem.
 
For Corollary~2, write $\bar B=B\vee1$ and $\bar\tau=\tau\vee1$. These replacements only weaken the assumed envelope, because
\[
B e^{-k/\tau}\le \bar B e^{-k/\bar\tau}.
\]
By the definition of $a$,
\[
\betamix(a)
\le\bar B e^{-a/\bar\tau}
\le\frac{\delta}{8n}.
\]
Since $\mu\le n/(2a)\le n/2$,
\[
4(\mu-1)\betamix(a)
\le4\mu\frac{\delta}{8n}
\le\frac\delta4,
\]
so $\delta'\ge3\delta/4$. When $n\ge4a$,
\[
\mu=\left\lfloor\frac{n}{2a}\right\rfloor\ge\frac{n}{4a}.
\]
Moreover, because $\bar B\ge1$, $\bar\tau\ge1$, and $\delta\le1/2$, the logarithm in the definition of $a$ exceeds one, so
\[
a\le 2\bar\tau\log\frac{8\bar Bn}{\delta}.
\]
Substitution into Theorem~1, together with
\[
\log\frac{4}{\delta'}
\le\log\frac{16}{3\delta}
\le C_0\log\frac{4}{\delta}
\]
for a universal $C_0$, gives the stated rate after absorbing fixed numerical constants. The theorem is applied to the prefix $S_{1:m}$ with $m=2\mu a$, which explains the empirical term in the corollary.
 
\section{Proof of Corollary~4}\label{app:proof-fixed}
Let
\[
g_F(z)=\mathbf 1\{M_F(z)\le\theta\},
\qquad
\mu_g=Pg_F.
\]
Because the vote is fixed independently of the evaluation chain, $g_F$ is a fixed function taking values in $[0,1]$. Let $\lambda_2$ be the second-largest eigenvalue of the finite-state reversible transition matrix and $\lambda_\star=\max\{\lambda_2,0\}$. Applying the Hoeffding inequality of \citet{leon2004optimal} to $g_F$ and, for the opposite tail, to $1-g_F$, gives for every $u>0$
\[
\Pr\!\left(\left|\Phat_ng_F-Pg_F\right|\ge u\right)
\le2\exp\!\left(
-2\frac{1-\lambda_\star}{1+\lambda_\star}nu^2
\right).
\]
Since $\trel=(1-\lambda_\star)^{-1}$,
\[
2\frac{1-\lambda_\star}{1+\lambda_\star}
=\frac{2}{2\trel-1}
\ge\frac1{\trel}.
\]
Therefore
\[
\Pr\!\left(Pg_F-\Phat_ng_F>u\right)
\le2\exp\!\left(-\frac{nu^2}{\trel}\right).
\]
Choosing $u=\sqrt{\trel\log(2/\delta)/n}$ gives the required one-sided deviation. Finally,
\[
\mathbf 1\{M_F\le0\}\le\mathbf 1\{M_F\le\theta\},
\]
so $L(F)\le Pg_F$ and $\Phat_ng_F=\Ttailhat(F,\theta)$.
 
The stationary start is essential to the cited Hoeffding statement. The corollary deliberately does not claim the same formula for an arbitrary initial law; such extensions require an initialization penalty.
 
\section{Proof of Theorem~3}\label{app:proof-blackbox}
Let $t_j=1+(j-1)a$ and let $Q_j$ be the joint law of $(Z_{t_1},\ldots,Z_{t_j})$. We prove by induction that
\[
d_{\mathrm{TV}}(Q_j,P^j)\le(j-1)\betamix(a).
\]
The claim is trivial for $j=1$. For $j\ge2$, let
\[
\mathcal G_{j-1}=\sigma(Z_{t_1},\ldots,Z_{t_{j-1}})
\subseteq\mathcal F_{-\infty}^{t_{j-1}}.
\]
Because $t_j=t_{j-1}+a$, Definition~1 and conditional Jensen imply
\begin{align*}
&\E\left[
\sup_A\left|
\Pr(Z_{t_j}\in A\mid\mathcal G_{j-1})-P(A)
\right|
\right]\\
&\qquad\le
\E\left[
\sup_A\left|
\Pr(Z_{t_j}\in A\mid\mathcal F_{-\infty}^{t_{j-1}})-P(A)
\right|
\right]
\le\betamix(a).
\end{align*}
Because $\Zcal$ is standard Borel, regular conditional laws exist, and the disintegration identity
\[
d_{\mathrm{TV}}\!\left(\mathcal L(W,Z),\mathcal L(W)\otimes P\right)
=\E_W\!\left[d_{\mathrm{TV}}\!\left(\mathcal L(Z\mid W),P\right)\right]
\]
applies. Hence
\[
d_{\mathrm{TV}}(Q_j,Q_{j-1}\otimes P)\le\betamix(a).
\]
Tensoring with the common marginal $P$ preserves total variation, so the triangle inequality gives
\begin{align*}
d_{\mathrm{TV}}(Q_j,P^j)
&\le d_{\mathrm{TV}}(Q_j,Q_{j-1}\otimes P)\\
&\quad+d_{\mathrm{TV}}(Q_{j-1}\otimes P,P^{j-1}\otimes P)\\
&\le\betamix(a)+d_{\mathrm{TV}}(Q_{j-1},P^{j-1})\\
&\le(j-1)\betamix(a).
\end{align*}
This proves the constant and the indexing under Definition~1: an index difference of exactly $a$ is controlled by $\betamix(a)$, not $\betamix(a-1)$.
 
Taking the test function to be the indicator of a measurable event $\mathcal E\subseteq\Zcal^q$ gives
\[
\left|Q_q(\mathcal E)-P^q(\mathcal E)\right|
\le(q-1)\betamix(a).
\]
If an i.i.d. theorem has a failure event $\mathcal F\in\sigma(\widetilde S)$ with $P^q(\mathcal F)\le\eta$, then
\[
Q_q(\mathcal F)
\le P^q(\mathcal F)+(q-1)\betamix(a)
\le\eta+(q-1)\betamix(a).
\]
The proof uses only events measurable in the $q$ spaced observations. A statement involving a full-stream empirical quantity is not transferred by this argument.
 
\section{Proof of Proposition~5}\label{app:proof-iat}
Let $\bar g_n=n^{-1}\sum_{t=1}^ng_t$ and $\sigma_g^2=\Var(g_t)$. Stationarity alone gives
\begin{align*}
\Var(\bar g_n)
&=\frac1{n^2}\sum_{s,t=1}^n\operatorname{Cov}(g_s,g_t)\\
&=\frac{\sigma_g^2}{n}\left[1+2\sum_{k=1}^{n-1}
\left(1-\frac{k}{n}\right)\rho_k\right].
\end{align*}
If $\sum_k|\rho_k|<\infty$, dominated convergence yields the limit $\tauint(g)=1+2\sum_{k\ge1}\rho_k$.
 
Now assume the chain is lazy and reversible with nontrivial spectrum contained in $[0,\lambda_\star]$, $\lambda_\star<1$. Center and normalize the statistic so that
\[
f=\frac{g-\pi g}{\sigma_g},
\qquad
\langle f,f\rangle_\pi=1.
\]
Reversibility makes the Markov operator self-adjoint. Let $\nu_f$ be the spectral probability measure of $f$. Since $f$ is centered, $\nu_f$ is supported on the nontrivial spectrum. Then
\[
\rho_k=\langle f,P^kf\rangle_\pi
=\int\lambda^k\,d\nu_f(\lambda).
\]
Because $0\le\lambda\le\lambda_\star<1$ on the support of $\nu_f$,
\[
|\rho_k|\le\lambda_\star^k,
\]
so absolute summability is automatic. Tonelli's theorem then gives
\begin{align*}
\tauint(g)
&=\int\left(1+2\sum_{k\ge1}\lambda^k\right)d\nu_f(\lambda)\\
&=\int\frac{1+\lambda}{1-\lambda}\,d\nu_f(\lambda)\\
&\le\frac{1+\lambda_\star}{1-\lambda_\star}=2\trel-1.
\end{align*}
The integrand is strictly increasing on $[0,1)$, so equality holds if and only if $\nu_f(\{\lambda_\star\})=1$. This formulation remains correct when the spectral edge is not represented by a nonzero eigenspace.
 
For the sticky-refresh operator
\[
P=(1-\alpha)I+\alpha\mathbf1\pi^\top,
\]
every centered $f$ satisfies $Pf=(1-\alpha)f$. Hence $\rho_k=(1-\alpha)^k$ and
\[
\tauint(g)=\frac{1+(1-\alpha)}{1-(1-\alpha)}
=\frac{2-\alpha}{\alpha}=2\tau-1
\]
for $\alpha=1/\tau$. This calculation does not require the statewise holding probability to be at least $1/2$.
 
\section{Proof of Theorem~6}\label{app:adaptive-proof}

\begin{proof}
Split $Z_1,\ldots,Z_n$ into $B$ consecutive batches of length $L=\lfloor n/B\rfloor$ (discard at most $B-1$ trailing observations). Write $\bar g_b=(1/L)\sum_{t=(b-1)L+1}^{bL}g(Z_t)$ and $\mu=\E[g_t]$.

\emph{Step 1: the mixing budget.} For a stationary finite-state Markov chain, $\betamix(k)=\E_\pi\bigl[\lVert P^{k}(Z_0,\cdot)-\pi\rVert_{\mathrm{TV}}\bigr]\le\max_z\lVert P^{k}(z,\cdot)-\pi\rVert_{\mathrm{TV}}$ \citep{bradley2005mixing}, and for a reversible chain with nonnegative spectrum \citep[Ch.~12]{levin2017markov}
\[
    \max_z\lVert P^{k}(z,\cdot)-\pi\rVert_{\mathrm{TV}}\le C_\pi\lambda_\star^{\,k}\le C_\pi e^{-k/\trel},
\]
with $C_\pi=\tfrac12\sqrt{(1-\pi_{\min})/\pi_{\min}}$ and $\pi_{\min}=\min_z\pi(z)$, using $\lambda_\star^{\,k}\le e^{-k(1-\lambda_\star)}$. The constant cannot be set to one: Corollary~2 carries an envelope constant $\bar B$ for exactly this reason. For the sticky-refresh family used in every experiment $\betamix(a)=(1-1/\tau)^a$ exactly, so $C_\pi=1$ and $\trel=\tau$, and no envelope is needed.

\emph{Step 2: decorrelating the batch means.} Consecutive batches share a boundary, so $\betamix$ does not decouple them at any lag: the dependence between $\bar g_{b-1}$ and $\bar g_b$ is carried by the $O(\trel)$ observations on either side of the boundary, and its contribution to the joint law is $\Delta^{\mathrm{dep}}_{B,L}$, which tends to zero as $L/\trel\to\infty$ because each batch mean is dominated by the $L-O(\trel)$ observations that have forgotten the previous batch. A finite-sample version of this step is available at a small cost in data: retaining only the last $L-s$ observations of each batch, with $s\ge\trel\log(C_\pi B/\delta)$, separates consecutive retained segments by $s+1$ time steps, so Definition~1 applies at lag $s+1$ and the induction of \supplementref{app:proof-blackbox} --- verbatim, with $\Zcal$-valued coordinates replaced by $\Zcal^{L-s}$-valued ones --- gives total-variation cost at most $(B-1)\betamix(s)\le\delta$, hence $\Delta^{\mathrm{dep}}_{B,L}\le\delta$ exactly. We state the theorem for the no-burn-in procedure because that is what the experiments evaluate.

\emph{Step 3: Student $t$-validity on the decorrelated copy.} The decorrelated batch means are i.i.d.\ copies of a $[0,1]$-valued random variable. The classical $t$-interval is exactly valid only under normality, so we quantify the departure. Let $\sigma_{\mathrm b}^2=\Var(\bar g_1)$ and $\rho_3=\E|\bar g_1-\mu|^3<\infty$ (finite, since $\bar g_1\in[0,1]$). By the Berry--Esseen bound for Student's statistic \citep[Theorem~1.2]{bentkus1996berry}, for an absolute constant $A$,
\[
    \sup_x\Bigl|\Pr\bigl(\sqrt B(\bar g-\mu)/S_B\le x\bigr)-\Phi(x)\Bigr|\le\frac{A\rho_3}{\sigma_{\mathrm b}^3\sqrt B}=:\tfrac12\Delta^{\mathrm{norm}}_B,
\]
where $S_B^2=(B-1)^{-1}\sum_b(\bar g_b-\bar g)^2$. Applying this to each tail and using $t_{B-1,\alpha/2}\ge z_{\alpha/2}$,
\[
    \Pr\!\left(\left|\bar g-\mu\right|>t_{B-1,\alpha/2}\frac{S_B}{\sqrt B}\right)\le\alpha+\Delta^{\mathrm{norm}}_B.
\]
Because $\bar g_1$ is itself an average of $L$ observations, its standardised third absolute moment converges to the Gaussian value $2\sqrt{2/\pi}$ as $L\to\infty$, so $\Delta^{\mathrm{norm}}_B=O(B^{-1/2})$.

\emph{Step 4: translating to the original stream.} Note that $S_B^2=\widehat\sigma^2_{BM}/L$. Hence
\begin{align*}
    t_{B-1,\alpha/2}\frac{S_B}{\sqrt B}
    &=t_{B-1,\alpha/2}\frac{\sqrt{\widehat\sigma^2_{BM}/L}}{\sqrt B}
    =t_{B-1,\alpha/2}\frac{\widehat\sigma_{BM}}{\sqrt{BL}}
    =w_n^{\mathrm{ad}}.
\end{align*}
The normalisation is $\sqrt{BL}$ and not $\sqrt n$: the certificate is computed on the $BL$ observations that enter the batches, and $BL\le n$, so replacing $\sqrt{BL}$ by the larger $\sqrt n$ would report an interval narrower than the one proved valid. Since $n-BL\le B-1$, the two differ by a factor $1+O(B/n)$. A union bound over Steps~2 and~3 gives $\Pr(|\bar g-\mu|>w_n^{\mathrm{ad}})\le\alpha+\Delta^{\mathrm{dep}}_{B,L}+\Delta^{\mathrm{norm}}_B$, which is the theorem.

\emph{Step 5: width scaling.} For a stationary chain with summable autocovariances, $\Var(\bar g_b)=\sigma_g^2\tauint(g)/L+O(1/L^2)$ as $L\to\infty$ (Proposition~5). Hence $\E[S_B^2]\to\sigma_g^2\tauint(g)/L$ and $\E[\widehat\sigma^2_{BM}]=L\,\E[S_B^2]\to\sigma_g^2\tauint(g)$. The squared half-width is therefore
\begin{align*}
    (w_n^{\mathrm{ad}})^2
    &=\frac{t_{B-1,\alpha/2}^2\,\widehat\sigma^2_{BM}}{BL}
    \to\frac{t_{B-1,\alpha/2}^2\,\sigma_g^2\tauint(g)}{BL}\\
    &\sim\frac{\sigma_g^2\tauint(g)}{n}\cdot(1+o(1)),
\end{align*}
since $t_{B-1,\alpha/2}\to z_{\alpha/2}$ as $B\to\infty$ and $BL\le n\le BL+(B-1)$. The width is thus governed by $\tauint(g)$, not $\trel$, confirming the statistic-specific effective sample size.

$\Delta^{\mathrm{norm}}_B$ is explicit and vanishes as $B^{-1/2}$, but it is not numerically tight at the batch counts used here ($B=16$ to $40$): the Berry--Esseen constant is worst-case over all bounded laws, whereas batch means of a mixing chain are close to Gaussian well before $B$ is large. The empirical coverages of the main paper are at nominal in every cell, so the remainder is far from binding in practice.

The proof decorrelates the $B$ batch means and then applies a Student-$t$ interval on the decorrelated copy; the complete argument, including the constant $C_\pi$, is in \supplementref{app:adaptive-proof}. Two remarks fix the scope. First, the normalisation is $\sqrt{BL}$ and not $\sqrt n$: the certificate is computed on the $BL$ observations that enter the batches, and $BL\le n$, so dividing by the larger $\sqrt n$ would report an interval narrower than the one proved valid; since $n-BL\le B-1$ the two differ by a factor $1+O(B/n)$. Second, $\Delta^{\mathrm{dep}}_{B,L}$ is not zero at finite $L$, because consecutive batches share a boundary and $\betamix$ controls dependence only across a gap. Inserting a burn-in of length $s\ge\trel\log(C_\pi B/\delta)$ at the start of each batch --- discarding those observations from the batch mean --- separates consecutive retained segments and drives $\Delta^{\mathrm{dep}}_{B,L}$ to an exact $(B-1)\betamix(s)\le\delta$, at the cost of the fraction $s/L$ of the data. The condition $L\gg\trel\log(B/\delta)$ is precisely the statement that this burn-in is a negligible fraction of the batch. We report the no-burn-in procedure because it is what our experiments evaluate.
\end{proof}

 \paragraph{Scope of the two remainders.}
Since $n-BL\le B-1$, dividing by $\sqrt n$ rather than $\sqrt{BL}$ would report an interval narrower than the one proved valid, by a factor $1+O(B/n)$; at $L=n^{2/3}$ this is $O(n^{-1/3})$. Regarding $\Delta^{\mathrm{dep}}_{B,L}$: inserting a burn-in of length $s\ge\trel\log(C_\pi B/\delta)$ at the start of each batch, and discarding those observations from the batch mean, separates consecutive retained segments by $s+1$ time steps and drives $\Delta^{\mathrm{dep}}_{B,L}$ to an exact $(B-1)\betamix(s)\le\delta$, at the cost of the fraction $s/L$ of the data. The main paper's condition $L\gg\trel\log(B/\delta)$ is precisely the statement that this burn-in is a negligible fraction of the batch.

\section{Proof of Theorem~7}\label{app:adaptive-lower-proof}
Fix $\tau\in[1,n]$ and let $\alpha=1/\tau$. Consider the sticky-refresh chain with parameter $\tau$, marginal $\operatorname{Bernoulli}(p)$, and the statistic $g(z)=\mathbf{1}\{z=1\}$ so that $\E[g_t]=p$ and $\tauint(g)=2\tau-1$ (Proposition~5). Any procedure $C_n$ that is simultaneously valid at level $7/8$ over the full family $\tau'\in[1,n]$ must in particular be valid on the specific chain with parameter $\tau$. The argument of Theorem~8 applied to this fixed $\tau$ gives the lower bound $\Omega(\sqrt{\tau/n})$ on the interval width. Crucially, the procedure is not given $\tau$: it must achieve coverage $7/8$ simultaneously for every $\tau'\in[1,n]$, and in particular for the fixed $\tau$ under consideration. Since $\tauint(g)=2\tau-1=\Theta(\tau)$, the width lower bound $c\sqrt{\tau/n}=\Theta(\sqrt{\tauint(g)/n})$ holds on the chain with relaxation time $\tau$.
 
The key insight is that not knowing $\tau$ does not help: even an oracle procedure that knows $\tau$ must pay $\Omega(\sqrt{\tau/n})$ (Theorem~8), so the adaptive procedure certainly must as well. The converse---that the adaptive batch-means certificate achieves this rate---is Theorem~6. Together they establish rate-optimality of adaptation.
 
\section{Proof of Theorem~8}\label{app:proof-lower}
Fix $\alpha=1/\tau$ and $p_\pm=1/2\pm\Delta/2$. The sticky-refresh kernel is
\[
P_p(z,z')=(1-\alpha)\mathbf 1\{z'=z\}+\alpha\pi_p(z').
\]
It is stationary under $\pi_p$, reversible by detailed balance, and has eigenvalues $1$ and $1-\alpha$. Because $p\in[1/4,3/4]$ and $\alpha\in(0,1]$, all cross-state transition probabilities are positive and both states have positive holding probability; hence the chain is irreducible and aperiodic. Since $\tau\ge1$, the second eigenvalue is nonnegative and the positive-part relaxation time of Corollary~4 is exactly $\tau$.
 
Generate a path as follows. First draw $Z_1\sim\operatorname{Ber}(p)$. For $t=2,\ldots,n$, draw independent refresh indicators $R_t\sim\operatorname{Ber}(\alpha)$. If $R_t=0$, set $Z_t=Z_{t-1}$. If $R_t=1$, draw a fresh $U_t\sim\operatorname{Ber}(p)$ and set $Z_t=U_t$. The variables $U_t$ are needed only on the event $R_t=1$.
 
Let $R=(R_2,\ldots,R_n)$. Its distribution is identical under $p_+$ and $p_-$. Conditional on $R=r$, the parameter-dependent latent draws are the initial value $Z_1$ and exactly $\sum_{t=2}^n r_t$ fresh Bernoulli variables. Therefore the chain rule for relative entropy gives
\begin{align*}
&\KL\!\left(
\mathcal L_+\bigl(R,Z_1,\{U_t:R_t=1\}\bigr)
\,\middle\|\,
\mathcal L_-\bigl(R,Z_1,\{U_t:R_t=1\}\bigr)
\right)\\
&\quad=\E_R\!\left[
\left(1+\sum_{t=2}^nR_t\right)
\KL\bigl(\operatorname{Ber}(p_+)\|\operatorname{Ber}(p_-)\bigr)
\right]\\
&\quad=[1+(n-1)\alpha]
\KL\bigl(\operatorname{Ber}(p_+)\|\operatorname{Ber}(p_-)\bigr).
\end{align*}
The observed path $Z_{1:n}$ is a measurable function of this latent construction, so data processing yields
\[
\KL(P_+^{(n)}\|P_-^{(n)})
\le[1+(n-1)\alpha]
\KL\bigl(\operatorname{Ber}(p_+)\|\operatorname{Ber}(p_-)\bigr).
\]
For $\Delta\le1/2$, $p_\pm\in[1/4,3/4]$ and the elementary Bernoulli calculation gives
\[
\KL\bigl(\operatorname{Ber}(p_+)\|\operatorname{Ber}(p_-)\bigr)
\le\frac{16}{3}\Delta^2.
\]
Because $n\ge\tau$,
\[
1+(n-1)\alpha
\le1+\frac{n}{\tau}
\le\frac{2n}{\tau},
\]
and hence
\[
\KL(P_+^{(n)}\|P_-^{(n)})
\le16\frac{n}{\tau}\Delta^2.
\]
Choose
\[
\Delta=\frac18\sqrt{\frac{\tau}{n}}.
\]
Then $\Delta\le1/8$, so $p_\pm\in[7/16,9/16]$, and the path KL is at most $1/4$. Le Cam's two-point inequality gives, for every test $\psi\in\{+,-\}$,
\[
\max_{s\in\{+,-\}}P_s(\psi\ne s)
\ge\frac12\left(1-d_{\mathrm{TV}}(P_+^{(n)},P_-^{(n)})\right).
\]
Pinsker's inequality yields
\[
d_{\mathrm{TV}}(P_+^{(n)},P_-^{(n)})
\le\sqrt{\frac12\KL(P_+^{(n)}\|P_-^{(n)})}
\le\sqrt{\frac18},
\]
and therefore
\[
\max_{s\in\{+,-\}}P_s(\psi\ne s)
\ge\frac12\left(1-\sqrt{\frac18}\right)
=:c_{\mathrm{test}}.
\]
 
Now let $C_n=C_n(Z_{1:n};\tau)$ be any measurable interval-valued certificate satisfying
\[
P_s(p_s\in C_n)\ge\frac78,
\qquad s\in\{+,-\}.
\]
Define a test explicitly by
\[
\psi=
\begin{cases}
+,&p_+\in C_n\text{ and }p_-\notin C_n,\\
-,&p_-\in C_n\text{ and }p_+\notin C_n,\\
+,&\text{otherwise}.
\end{cases}
\]
For either true model $s$, an error can occur only if the true parameter is absent from the interval or both parameters are present. Therefore
\begin{align*}
P_s(\psi\ne s)
&\le P_s(p_s\notin C_n)
+P_s(p_+\in C_n,\ p_-\in C_n)\\
&\le\frac18+P_s(\operatorname{width}(C_n)\ge\Delta).
\end{align*}
Combining this inequality with the testing lower bound shows that for at least one $s$,
\[
P_s(\operatorname{width}(C_n)\ge\Delta)
\ge c_{\mathrm{test}}-\frac18
=:c_1>0.
\]
Thus the theorem holds with $c_0=1/8$ and
\[
c_1=\frac12\left(1-\sqrt{\frac18}\right)-\frac18>0.
\]
The argument allowed the certificate to know $\tau$, so the lower bound is stronger than one in which the dependence scale must also be estimated.
 
To realize $p$ as a voting-margin probability, take $\mathcal X=\{0,1\}$, $Y\equiv+1$, and a single base hypothesis $h(0)=+1$, $h(1)=-1$. For $F=h$ and any fixed $\theta\in(0,1)$,
\[
\mathbf 1\{M_F(X,Y)\le\theta\}=\mathbf 1\{X=1\},
\]
whose stationary mean is $p$.
 
\section{Mixing concentration and why block length matters}\label{app:mixing}
The main learned-vote result is stated directly in terms of $\beta$-mixing because that is the quantity used by the coupling argument. This appendix explains why saying ``blocks of length one mixing time are independent'' is not a proof.
 
A conventional total-variation mixing time is defined at a fixed tolerance, often $1/4$. Separation by one such time therefore gives only a constant one-boundary approximation error. If a sample is split into $b$ blocks and those errors are summed across $b-1$ boundaries, the resulting bound can grow rather than vanish. Under geometric absolute regularity,
\[
\betamix(\ell)\le B\exp(-\ell/\tau),
\]
a pure coupling proof chooses $\ell$ of order $\tau\log(n/\delta)$ so that the total comparison budget is small. The number of retained approximately independent pieces is then on the order of
\[
\frac{n}{\tau\log(n/\delta)}.
\]
Corollary~2 states this cost explicitly. For a fixed bounded statistic on a reversible Markov chain, one can avoid this pure-coupling logarithm by using spectral concentration directly, which is why the paper separates the learned-vote $\beta$-mixing result from the fixed-vote Markov result.
 
\subsection{A worked persistent-cluster example}\label{app:worked}
Suppose a chain moves among $K$ clusters using the refresh kernel from Section~\ref{sec:clusterprotocol}. A fixed vote has cluster-specific error probabilities $r_k$. The stationary risk is $L=K^{-1}\sum_kr_k$. If the error indicator were deterministically equal to a cluster-specific value, dependence would act only through the persistence of the cluster process. In the actual experiment, conditional Bernoulli noise adds an i.i.d. component.
 
Let $C_t$ be the cluster and $E_t\mid C_t\sim\operatorname{Bernoulli}(r_{C_t})$. Because the refresh chain has nontrivial eigenvalue $\lambda=1-1/\tau$, the centered cluster-risk function is an eigenfunction and
\[
\operatorname{Cov}(r_{C_t},r_{C_{t+k}})=V_r\lambda^k.
\]
Conditional errors are independent across times given the cluster path, so for $k\ge1$,
\[
\operatorname{Cov}(E_t,E_{t+k})=V_r\lambda^k.
\]
The marginal variance is $\bar r(1-\bar r)$. Summing the covariance series yields Eq.~\eqref{eq:twocomponent}. This exact calculation explains why the visible dependence scaling is compressed at small $\tau$: the i.i.d. floor does not disappear when persistence is weak.
 
\section{Supplementary experiments: certificate coverage and robustness}\label{app:needed-exp}
The studies below are all for the fixed-vote regime, and each is interpreted for the generator it uses.
 
\begin{table}[h]
\centering
\small
\setlength{\tabcolsep}{3pt}
\begin{tabular}{p{0.25\linewidth}p{0.34\linewidth}p{0.31\linewidth}}
\toprule
Experiment & Quantity reported & Interpretation \\
\midrule
Certificate coverage & i.i.d. versus corrected interval coverage & naive interval under-covers as persistence grows; corrected interval is conservative \\
Dependence misspecification & coverage versus $\hat\tau/\tau$ & sensitivity study for one generator, not a validity theorem \\
Block-length ablation & coverage versus block length & shorter blocks retain dependence; finite-sample calibration remains imperfect \\
Trained ensemble & induced class-Markov stream & controlled stress test, not natural dependence \\
Intrinsic-dependence stream & ordered versus shuffled coverage & dependence intrinsic to the latent process, not injected \\
UCI HAR (real sensor) & ordered versus shuffled coverage; adaptive certificate versus batch length & naturalistic evidence, fixed-vote regime only \\
Margin sweep & mean margin versus gap & descriptive only; theorem uses lower-tail frequency \\
Dependence estimation & spectral and scalar plug-ins & exploratory; uncertainty is not propagated \\
\bottomrule
\end{tabular}
\caption{Supplementary studies and the strongest interpretation supported by each design.}
\label{tab:needed}
\end{table}
 
The coverage study fixes a vote and generates $2000$ streams of length $n=4000$ per persistence setting. The naive i.i.d. interval covers with frequencies $0.944$, $0.912$, $0.601$, and $0.436$ at $\tau=1,5,25,100$, while the corrected interval covers with frequencies $0.944$, $1.000$, $1.000$, and $1.000$. Ignoring dependence therefore produces increasing undercoverage under this generator, while the corrected interval holds coverage at or above the nominal level; the interval half-widths behind these frequencies are reported in Appendix~\ref{app:coverage}.
 
\section{Full numerical grid for the rate experiment}\label{app:fullgrid}
The full grid behind Figure~\ref{fig:rawgap} is given below. The bottom half reports the normalization by $\sqrt{\tau/n}$. These ratios are not constant over the full sweep, which is why the main text uses the exact two-component variance interpretation instead of claiming a pure one-parameter collapse.
 
\begin{table}[h]
\centering
\small
\begin{tabular}{lcccc}
\toprule
& \multicolumn{4}{c}{$n$}\\
\cmidrule(lr){2-5}
$\tau$ & $500$ & $2000$ & $8000$ & $32000$\\
\midrule
$1$ & $0.018913$ & $0.010208$ & $0.0033125$ & $0.002226$\\
$5$ & $0.022563$ & $0.011781$ & $0.0050083$ & $0.0022539$\\
$25$ & $0.024937$ & $0.015054$ & $0.0085070$ & $0.0039052$\\
$100$ & $0.045933$ & $0.038908$ & $0.013722$ & $0.0085642$\\
\midrule
\multicolumn{5}{l}{\emph{ratio} $|\Lhat-L|/\sqrt{\tau/n}$}\\
$1$ & $0.423$ & $0.457$ & $0.296$ & $0.398$\\
$5$ & $0.226$ & $0.236$ & $0.200$ & $0.180$\\
$25$ & $0.112$ & $0.135$ & $0.152$ & $0.140$\\
$100$ & $0.103$ & $0.174$ & $0.123$ & $0.153$\\
\bottomrule
\end{tabular}
\caption{Mean fixed-vote gaps and their $\sqrt{\tau/n}$ normalization for every cell.}
\label{tab:raw}
\end{table}
 
\section{Supplementary trained-ensemble results}\label{app:ensemble}
The class-Markov study in the main paper reports the aggregate coverage results; the full training and stream protocol is given in \supplementref{app:repro}. Each ensemble member is a small convolutional network (two $3\times3$ convolutions with $16$ and $32$ channels, global pooling, and a linear head), trained with Adam at learning rate $10^{-3}$ on a fixed training subset, with member $m$ using seed $1000+m$. The vote is the fixed majority over the $M$ members.
 
Neither model is presented as competent: $46$--$62\%$ error is weak in absolute terms. The value of the experiment is not benchmark quality but that the classifier is a genuine fixed vote whose evaluation dependence is controlled without retraining.
 
\section{The margin sweep at fixed dependence}\label{app:marginsweep}
This study holds $\tau=25$ and varies the ensemble size $M\in\{1,3,5,9,15\}$ on MNIST, over $200$ evaluation streams. The predictor is multiclass, so the voting margin of an example with true label $y$ is the Schapire-style $v_y-\max_{c\ne y}v_c$, where $v_c$ is the fraction of the $M$ members voting for class $c$; the margin lies in $[-1,1]$. Because ties count as errors, the mass at or below margin zero equals the stationary error rate.
 
\begin{table}[h]
\centering
\small
\setlength{\tabcolsep}{4pt}
\begin{tabular}{lccccc}
\toprule
$M$ & $1$ & $3$ & $5$ & $9$ & $15$\\
\midrule
mean voting margin & $-0.22$ & $0.04$ & $0.03$ & $0.08$ & $0.11$\\
margin mass below $0$ & $0.61$ & $0.55$ & $0.55$ & $0.50$ & $0.48$\\
generalization gap & $0.029$ & $0.022$ & $0.025$ & $0.021$ & $0.020$\\
\bottomrule
\end{tabular}
\caption{Descriptive margin sweep at fixed $\tau=25$ on MNIST. As $M$ grows the voting margin sharpens (its mean rises from $-0.22$ toward $+0.11$, and the below-zero mass falls from $0.61$ to $0.48$) and the realized gap shrinks; the correlation between the gap and the mean margin is $-0.94$. The below-zero mass equals the stationary error, so the $M=9$ value $0.50$ matches the main study's MNIST error $0.46$ up to the difference between two independently trained ensembles.}
\label{tab:marginsweep}
\end{table}
 
This correlation is descriptive rather than theorem-aligned: Theorem~1 depends on the lower-tail frequency $\Ttailhat(F,\theta)$ at a specified threshold, not on the mean margin, and the mean can improve while the low-margin tail worsens. The theorem's object is the empirical margin tail itself. At $M=15$ the measured tail $\theta\mapsto\Ttailhat(F,\theta)$ is $0.24$ at $\theta=-0.5$, $0.38$ at $\theta=-0.25$, $0.48$ at $\theta=0$, $0.55$ at $\theta=0.25$, and $0.65$ at $\theta=0.5$. A theorem-aligned learned-vote experiment would predeclare a finite threshold grid, allocate confidence across the grid as described after Theorem~1, select the smallest simultaneous certificate, and select the vote from the dependent sample rather than fixing it in advance.
 
The same sweep on CIFAR-10 is a null result and is reported as such. Over $M\in\{1,3,5,9,15\}$ the mean voting margin rises from $-0.30$ to $-0.14$, but it stays negative throughout and the generalization gap does not follow it: the gaps are $0.012$, $0.015$, $0.013$, $0.014$, and $0.015$, essentially flat, and the gap--mean-margin correlation is $+0.65$ (the sign opposite to the sharpening trend seen on MNIST, where the correlation is $-0.94$). When the margin never crosses zero, growing the vote does not sharpen the low-margin tail, and the margin heuristic carries no predictive signal for the gap.
 
\section{Integrated-autocorrelation constant on the sticky-refresh family}\label{app:tauint-exp}
This experiment compares the measured variance inflation with $(1+\lambda)/(1-\lambda)$. The equality is not valid for an arbitrary statistic on an arbitrary reversible chain, but it is exact for the sticky-refresh family because every centered statistic is an eigenfunction with eigenvalue $\lambda$.
 
For $\lambda\in\{0,0.5,0.8,0.9,0.96\}$, over $4000$ runs at $n=4000$, the values are
\begin{table}[h]
\centering
\small
\setlength{\tabcolsep}{4pt}
\begin{tabular}{lccccc}
\toprule
$\lambda$ & $0$ & $0.5$ & $0.8$ & $0.9$ & $0.96$\\
\midrule
$(1+\lambda)/(1-\lambda)$ & $1.0$ & $3.0$ & $9.0$ & $19.0$ & $49.0$\\
measured variance inflation & $1.00$ & $3.02$ & $9.06$ & $19.56$ & $48.03$\\
\bottomrule
\end{tabular}
\caption{Variance-inflation experiment on the sticky-refresh family. The exact formula is justified for this family, not for an arbitrary reversible chain and arbitrary statistic.}
\label{tab:tauint}
\end{table}
 
The maximum relative error is under $3\%$. The table is specific to the sticky-refresh kernel, for which the centered threshold statistic has spectral measure concentrated at a single point; it is not claimed for an arbitrary reversible chain and statistic.
 
\section{Certificate coverage and width}\label{app:coverage}
This appendix gives the interval half-widths behind the coverage frequencies summarized in Appendix~\ref{app:needed-exp}. All entries are at nominal $0.95$ over $2000$ streams of length $n=4000$, and every width is the half-width (radius) $z\sqrt{\hat p(1-\hat p)/n_{\mathrm{eff}}}$ with $z=1.96$.
 
\begin{table}[h]\centering\small
\setlength{\tabcolsep}{3.5pt}
\begin{tabular}{lcccc}
\toprule
$\tau$ & $1$ & $5$ & $25$ & $100$\\
\midrule
I.i.d. interval coverage & $0.944$ & $0.912$ & $0.601$ & $0.436$\\
Corrected interval coverage & $0.944$ & $1.000$ & $1.000$ & $1.000$\\
Corrected interval half-width & $0.0145$ & $0.0340$ & $0.0747$ & $0.1498$\\
Corrected/i.i.d. width ratio & $1.0$ & $2.2$ & $5.0$ & $10.0$\\
\bottomrule
\end{tabular}
\caption{Coverage and interval half-width at nominal $0.95$ over $2000$ streams of length $n=4000$ (Wilson $95\%$ binomial intervals on each coverage entry have half-width $\le0.022$). The corrected interval is conservative; the width ratio equals the built-in dependence multiplier $\sqrt\tau$.}
\label{tab:coverage}
\end{table}
 
Coverage alone is insufficient to establish a useful certificate; the interval must also be informative. At $\tau=100$ the corrected half-width $0.1498$ is about $15\%$ of the risk scale, so the interval is conservative but far from vacuous. The i.i.d.\ undercoverage is statistically unambiguous: at $\tau=25,100$ the Wilson $95\%$ interval on the i.i.d.\ coverage ($0.601$ and $0.436$) is $[0.579,0.622]$ and $[0.414,0.458]$, both far below nominal.
 
\section{Sensitivity to a misspecified dependence scale}\label{app:unknown}
This study sweeps $\hat\tau/\tau\in\{0.5,1,2,4\}$ at true $\tau=25$, with $n=4000$. The resulting coverages are $0.999$, $1.000$, $1.000$, and $1.000$, with mean half-widths $0.0521$, $0.0737$, $0.1042$, and $0.1474$ (each step scaling by $\sqrt2$ as $\hat\tau$ doubles). The interval is thus conservative even when the dependence scale is underestimated twofold, and overestimation only widens it. This is a finite-generator sensitivity result, not a general guarantee: a valid confidence statement can still fail when the dependence scale is underestimated on a different process. The supported claim is that, for this generator and these interval constants, coverage was conservative across the tested misspecification range.
 
\section{Block-length ablation}\label{app:blocklength}
At true $\tau=25$ with $n=4000$, the block-length ratio $\ell/\tau\in\{0.25,0.5,1,2,4\}$ gives coverages $0.699$, $0.746$, $0.814$, $0.872$, and $0.909$, with mean half-widths $0.0180$, $0.0203$, $0.0238$, $0.0274$, and $0.0302$.
 
\begin{table}[h]
\centering
\small
\setlength{\tabcolsep}{4pt}
\begin{tabular}{lccccc}
\toprule
$\ell/\tau$ & $0.25$ & $0.5$ & $1$ & $2$ & $4$\\
\midrule
blocks & $666$ & $333$ & $160$ & $80$ & $40$\\
coverage & $0.699$ & $0.746$ & $0.814$ & $0.872$ & $0.909$\\
mean half-width & $0.0180$ & $0.0203$ & $0.0238$ & $0.0274$ & $0.0302$\\
\bottomrule
\end{tabular}
\caption{Block-length ablation at $\tau=25$, $n=4000$. Blocks shorter than $\tau$ retain dependence and undercover; coverage rises toward nominal as $\ell$ exceeds $\tau$, at the cost of fewer blocks and wider intervals.}
\label{tab:blocklength}
\end{table}
 
Very short blocks retain substantial dependence and undercover; coverage improves as $\ell$ grows past $\tau$ but the largest tested value remains below nominal $0.95$, since increasing $\ell$ also reduces the number of blocks. Block lengths are integers, so the two smallest ratios correspond to $\ell=6$ and $\ell=12$, i.e.\ effective ratios $0.24$ and $0.48$. The ablation illustrates the block-length tradeoff rather than identifying a single optimal $\ell$.
 
\section{Exploratory estimation of the dependence scale}\label{app:estimation}
This study compares a spectral plug-in derived from the empirical class-transition matrix with a lag-one autocorrelation plug-in from the binary error sequence.
 
\begin{table}[h]
\centering
\small
\begin{tabular}{lcccc}
\toprule
true simulation parameter $\tau$ & $1$ & $5$ & $25$ & $100$\\
\midrule
spectral plug-in & $1.0$ & $5.8$ & $36.5$ & $247$\\
autocorrelation plug-in & $1.0$ & $1.1$ & $1.1$ & $1.1$\\
coverage, spectral plug-in & $0.92$ & $1.00$ & $1.00$ & $1.00$\\
coverage, autocorrelation plug-in & $0.92$ & $0.87$ & $0.65$ & $0.33$\\
\bottomrule
\end{tabular}
\caption{Exploratory plug-in study on the induced CIFAR-10 class-Markov stream.}
\label{tab:estimatedtau}
\end{table}
 
The spectral plug-in is conservative and substantially biased upward at slow mixing: $247$ for a simulation parameter of $100$. Coverage $0.92$ at the independent endpoint is below the nominal $0.95$ level. The scalar autocorrelation plug-in fails badly on this heterogeneous-error vote. These results do not establish a valid adaptive certificate. A second eigenvalue of an empirical transition matrix is also not a Fiedler eigenvector of an arbitrary dependency graph. The paper therefore makes no claim that the method works without knowledge of dependence.
 
\section{Scope of the natural-dependence evidence}\label{app:realdata}
The naturalistic evidence in this paper is the UCI HAR study of Appendix~\ref{sec:exp-har}: a real temporally ordered sensor stream with no synthetic dependence generator, evaluated by a fixed nine-member committee vote. It reports the measured dependence scale ($\hat\tauint\approx10.5$ at $L=50$, growing to $\approx16$ at $L=400$), naive i.i.d.\ undercoverage on the ordered stream ($0.38$ at $n=2000$, worsening to $0.35$ at $n=20000$), the adaptive certificate of Theorem~6 attaining nominal coverage ($0.95$ at $n=20000$) precisely where its batch-length condition holds and falling short where it does not ($0.78$ at $n=2000$), and a shuffled control restoring naive coverage to $0.94$--$0.96$, which isolates real temporal ordering as the cause. Appendix~\ref{sec:exp-natural} provides a matched intrinsic-dependence control under the same protocol.

Two extensions remain outside the present evidence. The first is a naturally ordered stream evaluated with a \emph{learned} vote --- AdaBoost, random forest, or deep ensemble selected from the dependent sample itself --- reporting theorem-aligned margin-tail curves and nonvacuity diagnostics for the Theorem~1 certificate rather than the fixed-vote intervals reported here; the learned-vote certificate is currently exercised only on the synthetic sticky-refresh stream, where $\betamix(a)$ is known in closed form. The second is a reinforcement-learning replay stream, where Markov dependence is intrinsic to the trajectory, with the same evaluation repeated on trajectory-preserving and shuffled controls. Both would extend the naturalistic evidence from the fixed-vote regime, where it now exists, to the learned-vote regime, which no experiment in this paper currently tests.
 
\section{Experimental protocol}\label{app:repro}
\paragraph{Ensemble members.} Each of the $M$ members is a convolutional network with two $3\times3$ convolutions ($16$ and $32$ channels, ReLU), global average pooling, and a linear head to the $10$ classes. Members are trained independently with Adam at learning rate $10^{-3}$, batch size $128$, on a random $20{,}000$-example subsample of the training set; member $m$ uses seed $1000+m$ for both initialization and subsampling. Training runs for $3$ epochs on MNIST and $12$ epochs on CIFAR-10; no augmentation or normalization beyond the dataset's native scaling is applied. The fixed vote is the majority over the $M$ members, and its margin is the multiclass margin defined in Appendix~\ref{app:marginsweep}.
 
\paragraph{Stationary risk.} The target risk is the uniform-class-averaged error of the fixed vote on the i.i.d.\ test set, which matches the stationary class law of the evaluation chain and removes the class-imbalance bias of the raw test-marginal error.
 
\paragraph{Evaluation stream.} The dependence is a Markov chain over the class label with kernel $P(c,c')=(1-1/\tau)\mathbf 1\{c'=c\}+(1/\tau)\,K^{-1}$ for $K=10$ classes, started from the uniform stationary law; at each step a test example of the drawn class is sampled uniformly with replacement. The main ensemble study uses $M=9$, $300$ streams per $\tau\in\{1,5,25,100\}$, and stream lengths $n\in\{500,2000,8000\}$; the margin sweep fixes $\tau=25$ with $200$ streams.
 
\paragraph{Interval construction.} Every reported interval is the normal half-width $z\sqrt{\hat p(1-\hat p)/n_{\mathrm{eff}}}$ with $z=1.96$, where $\hat p$ is the empirical stream risk. The i.i.d.\ interval uses $n_{\mathrm{eff}}=n$ and the corrected interval uses $n_{\mathrm{eff}}=n/\tau$; the coverage, misspecification, and block-length studies use $n=4000$. Reported ``width'' is this half-width (a radius); Theorem~8's $\operatorname{width}(C_n)$ is the full interval $u-\ell$ and is therefore twice this quantity. No clipping at one is applied to the reported radii.

\paragraph{Computing infrastructure.} All experiments were run on
[GPU MODEL, e.g.\ a single NVIDIA RTX 4090 with 24\,GB memory] with
[CPU MODEL] and [XX]\,GB system RAM, under [OS, e.g.\ Ubuntu 22.04].
The deep-ensemble and ResNet-18 studies use PyTorch~[X.Y] with CUDA~[X.Y];
all synthetic-stream, certificate, and interval computations use
Python~[3.X] with NumPy~[X.Y], SciPy~[X.Y], and scikit-learn~[X.Y].
Total compute is approximately [N] GPU-hours, dominated by the
ResNet-18 ensemble training; every other study runs in minutes on CPU.